\documentclass[11pt]{article}
\usepackage[letterpaper,margin=1in]{geometry}

\usepackage{amsmath,amssymb,amsthm}
\usepackage{mathtools}
\usepackage{xcolor}
\usepackage{xspace}
\usepackage{graphicx}
\usepackage[numbers,round]{natbib}
\usepackage[hyphens]{url}
\usepackage{hyperref}
\usepackage[nameinlink,capitalise]{cleveref}
\usepackage{authblk}

\crefname{equation}{Eq.}{Eq.}
\Crefname{equation}{Eq.}{Eq.}
\crefname{lemma}{Lemma}{Lemma}
\Crefname{lemma}{Lemma}{Lemma}
\crefname{theorem}{Thm.}{Thm.}
\Crefname{theorem}{Thm.}{Thm.}
\crefname{section}{Sec.}{Sec.}
\Crefname{section}{Sec.}{Sec.}
\crefname{subsection}{Sec.}{Sec.}
\Crefname{subsection}{Sec.}{Sec.}
\crefname{assumption}{Assumption}{Assumption}
\Crefname{assumption}{Assumption}{Assumption}

\newtheorem{theorem}{Theorem}
\newtheorem{lemma}{Lemma}
\newtheorem{corollary}{Corollary}
\newtheorem{definition}{Definition}
\newtheorem{assumption}{Assumption}
\theoremstyle{remark}
\newtheorem{remark}{Remark}
\newtheorem{example}{Example}

\newcommand*\dif{\mathop{}\!\mathrm{d}}
\newcommand{\DataDistri}{\ensuremath{\mathcal{D}}\xspace}
\newcommand{\Measures}{\ensuremath{\mathcal{M}_1^+}\xspace}
\newcommand{\E}{\ensuremath{\mathbb{E}}\xspace}
\newcommand{\R}{\ensuremath{\mathbb{R}}\xspace}
\newcommand{\X}{\ensuremath{\mathcal{X}}\xspace}
\newcommand{\Y}{\ensuremath{\mathcal{Y}}\xspace}
\newcommand{\KL}{\ensuremath{\operatorname{KL}}\xspace}
\newcommand{\DataSpace}{\ensuremath{\Omega}\xspace}
\newcommand{\DataSpaceBorel}{\ensuremath{\mathcal{B}(\Omega)}\xspace}
\newcommand{\SampleSpace}{\ensuremath{\DataSpace^N}\xspace}
\newcommand{\SampleSpaceBorel}{\ensuremath{\mathcal{B}(\DataSpace)^{\otimes N}}\xspace}
\newcommand{\learner}{\ensuremath{\mathcal{A}}\xspace}
\newcommand{\Hypo}{\ensuremath{\mathcal{H}}\xspace}
\newcommand{\HypoBorel}{\ensuremath{\mathcal{F}}\xspace}
\newcommand{\rhoAlways}{\ensuremath{\rho_{\mathrm{always}}}\xspace}
\newcommand{\rhoNever}{\ensuremath{\rho_{\mathrm{never}}}\xspace}
\newcommand{\rhoDeployed}{\ensuremath{\rho_{\textsc{lupi}}}\xspace}

\title{Quantifying the Value of Privileged Information\\ Using a PAC-Bayesian Approach}

\author[1,2,3]{Vasily Bokov}
\author[3]{Sebastian Schmitt}
\author[1,2]{Vedran Dunjko}
\author[1,2]{Hao Wang}

\affil[1]{$\langle$aQa$^{\mathrm{L}}\rangle$, Leiden University, The Netherlands}
\affil[2]{LIACS, Leiden University, Leiden, The Netherlands}
\affil[3]{Honda Research Institute Europe GmbH, Offenbach, Germany}

\date{}
\begin{document}
\maketitle

\begin{abstract}
In practice, various learning scenarios provide access to auxiliary features exclusively during training. Incorporating such data to enhance model performance gave rise to a paradigm known as \textit{Learning Using Privileged Information} (LUPI).
While this extra information is intended to improve the resulting model, establishing a generalized, cohesive understanding of how privileged information (PI) transfers useful knowledge remains a challenge.
Vapnik's original theory and subsequent works offer performance guarantees in certain cases, but these results are inherently per-algorithm and rely on setting-specific proof approaches. Consequently, a more general framework explaining how and when PI transfers useful knowledge is still missing.
To bridge this gap, we introduce an algorithm-agnostic, information-theoretic approach based on the PAC-Bayes framework. Rather than asking whether a particular algorithm exploits PI, we ask how much value it could offer: comparing the tightest achievable risk bound with and without PI yields its \emph{potential} --- an upper limit on the extractable gain. We introduce a metric that quantifies this potential directly from empirical training risk, bypassing the need for test-time data access, and validate our findings in both supervised and unsupervised settings. The results demonstrate a robust correspondence between our training-time metric and true test-time performance gains. Ultimately, this work takes a necessary step toward an information-theoretic understanding of LUPI, and quantifying the potential of privileged features before committing to a model.
\end{abstract}

\section{Introduction}

When collecting data to train machine learning models, we often find ourselves in a position of informational asymmetry.
During training, we might have access to high-fidelity or expensive auxiliary features that are completely unavailable when the model is deployed in production.
For example, autonomous agents can be trained using ground-truth simulation parameters that real-world onboard sensors cannot observe at runtime \citep{chen2020learning, learning_by_cheating}.
Also, clinical medical diagnostics can utilize comprehensive laboratory biopsies to train an image classifier, yet the deployed model must perform rapid screening using only standard non-invasive diagnostic tools \citep{Gao2024}.
The Learning Under Privileged Information (LUPI) paradigm provides a formal framework for this setting, aiming to exploit these transient training-only features, termed Privileged Information (PI), to guide the training process toward a superior test-time predictor.
However, outside of a few specific domains, we rarely know whether models use this extra information efficiently, or if a given auxiliary feature will genuinely help the learning.

Since its introduction \citep{Vapnik2009}, all LUPI theoretical guarantees have been setting-specific.
Each proof heavily exploited the choice of a particular learning model
(e.g. Support Vector Machines \citep{JMLR:v16:vapnik15b, Vapnik2009} or custom regularized deep networks \citep{PI-DUAL, PI-LLM}) rather than the presence of PI itself.
Moreover, these frameworks almost entirely evaluate the utility of PI through the lens of computational or sample-complexity advantages, proving, for example, that PI can accelerate optimization convergence rates \citep{PI_deepNN, JMLR:v16:vapnik15b} or demonstrating computational separations when auxiliary features are derived from quantum hardware \citep{LUQPI}.

What is missing is an algorithm-agnostic account that isolates how the presence of PI reshapes the underlying learning problem from a pure, information-theoretic standpoint. Previous works conflate two questions: whether PI makes the learning problem easier, and whether a chosen deployment rule preserves that advantage when PI is unavailable at inference. We separate and answer both within the PAC-Bayes framework. Starting from Catoni's PAC bound~\citep{2007}, we introduce two idealized endpoint posteriors: the never-PI Gibbs measure (PI never available) and an oracle, or always-PI, measure on LUPI's hypothesis space (PI available in training and inference). Their true-risk gap quantifies the \emph{potential} of PI; since PI is absent at inference, a LUPI learner must push the oracle posterior through a data-independent deployment map $T$, which quantifies its \emph{realization}. We use Gibbs measures for both endpoints because each uniquely minimizes and simplifies its Catoni upper bound, giving the tightest certificate from which we can gauge the true-risk gap. This lets us define a utility indicator for PI --- the log-partition gap between the never-PI and oracle posteriors, which can be evaluated from training quantities as prior integrals without sampling either posterior. We further characterize when this potential can be realized through a deployment map $T$, providing a deployment-certificate criterion. Our contributions are:
\begin{itemize}
\item We develop an algorithm-agnostic PAC-Bayes formulation to certify the potential of PI, in which we propose to measure the potential with the log-partition function difference between two idealized Gibbs endpoint posteriors, independently of the learning algorithm. We also provide the condition at which the deployment map can realize the potential.
\item We extend the scope of our theoretical results from canonical Gibbs posteriors to an important subclass of exponential-family posteriors, covering exact Bayesian posteriors for linear, generalized-linear, and Gaussian-process models.
\item We validate the theoretical results in a Gaussian mixture model (GMM) with Bayesian estimation and Gaussian process classification (GPC) with privileged noise, where the log-partition function difference metric confidently tracks the realized test-time gain as PI quality varies.
\end{itemize}

\section{Background} \label{sec:background}
Let $\X$ and $\Y$ be topological spaces with their Borel
$\sigma$-algebras. Let $\DataSpace\coloneqq\X\times\Y$,
$\DataSpaceBorel\coloneqq\mathcal B(\X)\otimes\mathcal B(\Y)$, and denote by
$\DataDistri\in\mathcal P(\DataSpace)$ the
probability measures on $\Omega$.
The sample $S=((X_i,Y_i))_{i=1}^N\sim\DataDistri^{\otimes N}$ takes values in
$(\SampleSpace,\SampleSpaceBorel)$; write $s=((x_i,y_i))_{i=1}^N$ for a
realization. Let $(\Hypo,d_\Hypo)$ be a Polish space of Borel-measurable maps
$h\colon\X\to\R$, with Borel $\sigma$-algebra equal to the evaluation
$\sigma$-algebra
$\HypoBorel=\sigma(\{\operatorname{ev}_x^{-1}(A):x\in\X,\,
A\in\mathcal B(\R)\})$, where $\operatorname{ev}_x(h)=h(x)$. For any
$h_0\in\Hypo$, we consider $\mathcal M_1^+(\Hypo)=\{\rho\in\mathcal P(\Hypo):
\int d_\Hypo(h,h_0)\rho(\dif h)<\infty\}$. Finally, let $\ell\colon\R\times\Y\to[0,1]$ be Borel measurable and
assume $(h,x,y)\mapsto\ell(h(x),y)$ is jointly measurable.

\subsection{Learning Using Privileged Information} \label{sec:LUPI}
The Learning Under Privileged Information (LUPI) framework~\citep{Vapnik2009,JMLR:v16:vapnik15b} assumes that, during training, the learner observes an additional PI variable $x^*$ taking values in a PI space $\mathcal{X}^*$. This information is unavailable at
inference, so the final predictor must still be a function $h\colon \X\to\R$ of the ordinary input alone.
For a topological PI space $\X^*$, we set
$\DataSpace^* \coloneqq \X\times\X^*\times\Y$ with product Borel $\sigma$-algebra
$\mathcal B(\X)\otimes\mathcal B(\X^*)\otimes\mathcal B(\Y)$, and let
$\DataDistri^*$ be a probability measure whose $(X,Y)$-marginal is
$\DataDistri$. Then $S^*=((X_i,X_i^*,Y_i))_{i=1}^N\sim
(\DataDistri^*)^{\otimes N}$.

PI can reveal latent structure, example difficulty, or measurement quality, and thereby guide training towards a better predictor; it changes how $h$ is learned,
but not what information $h$ may use after deployment
(see~\cref{sec:deployment_mape}). The canonical example is SVM+~\citep{Vapnik2009,JMLR:v16:vapnik15b}, in which the
slack variable $\xi_i$ of each training point $x_i$ is modelled as a \emph{correcting function} of the privileged variable $x_i^*$, i.e., $\xi_i=\langle w^*,\phi^*(x_i^*)\rangle+b^*$, where $\xi_i$ measures how difficult the $i$-th example is to classify --- the margin violation of a \emph{training} point --- a quantity that simply does not exist at prediction time, so the correcting function has no role to play once training finishes (also see~\cref{sec:nontrivial-deployment-examples}).

Throughout, we treat PI as an intrinsic property of the data rather than of the learner.

\begin{assumption}[PI is a property of the data]\label{ass:pi-data}
The privileged variable is generated independently of the model: for any hypothesis/weights $w$,
$$
P(x^*\mid x, w) = P(x^*\mid x).
$$
Equivalently, the conditional law of $X^*$ given $X$ does not depend on the learner. This makes the conditional kernel $x\mapsto P(X^*\in\cdot\mid X=x)$ a fixed, data-independent object that we may use when constructing the deployment map (\cref{sec:deployment_mape}).
\end{assumption}

Subsequent work has extended LUPI beyond explicit slack correction.
MIML-FCN+~\citep{YangZC+17} introduced a deep analogue of SVM+, using a
privileged network to predict the loss of the ordinary-input network.
Generalized distillation~\citep{LopezPaz2016} instead trains a teacher on
privileged inputs and transfers its soft predictions to a student defined on
$\X$.
A complementary line uses PI to model training labels' uncertainty:
GPC+ represents privileged information through input-dependent label
noise~\citep{GPC+}, while heteroscedastic dropout~\citep{LambertSS18} makes the variance of Gaussian dropout a learned function of $x^*$. More recently, TRAM~\citep{CollierJK+22} transfers PI through a shared representation with privileged and ordinary prediction heads, training the ordinary head to approximate the prediction obtained after marginalizing the unavailable information.

\subsection{PAC-Bayes Framework} \label{sec:PAC-Bayes-basics}
A \emph{prior} $\pi\in\mathcal{M}_1^+(\Hypo)$ is fixed independently of the training samples $S$. A randomized learning rule is a Markov kernel
$
\learner \colon\SampleSpace \times \HypoBorel \to [0,1],
$
such that for every $s\in\SampleSpace$, the map $B\mapsto\learner(s,B)$ is a probability measure on $(\Hypo,\HypoBorel)$, and for every $B\in \HypoBorel$, the map $s\mapsto\learner(s,B)$ is $\SampleSpaceBorel$-measurable.
For random sample $S$, $\rho_S\coloneqq\learner(S,\cdot)$ is a random measure on $(\Hypo, \HypoBorel)$.

We define the \emph{true risk} and the \emph{empirical risk}:
\begin{align}
    R(h)   &= \mathbb{E}_{(X,Y) \sim \mathcal{D}}\!\left[\ell(h(X), Y)\right],
             & \text{(true risk)} \\
    r_S(h) &= \frac{1}{N}\sum_{i=1}^N \ell(h(x_i), y_i).
           & \text{(empirical risk)}
\end{align}
Thus $r_S$ is a random functional and $r_s$ is the corresponding realization. We write $r$ when the sample/realization is clear from the context. For
integrable $f$, denote by $\langle f\rangle_\rho=\int f(h)\rho(\dif h)$ the average function value. In particular, $\langle R\rangle_\rho = \mathbb{E}_{h \sim \rho}[R(h)]$ and $\langle r\rangle_\rho = \mathbb{E}_{h \sim \rho}[r(h)]$ are the true
and empirical risks of $h\sim\rho$. When $\rho=\rho_S$, both are random because the posterior depends on the sample. Catoni's theorem controls this sample-dependence.

\begin{theorem}[\citealp{2007}]
\label{thm:catoni}
Let $\pi\in\mathcal{P}(\Hypo)$ be independent of the sample. For any fixed
inverse-temperature $\lambda>0$ and error rate $\varepsilon\in(0,1)$, with probability at least $1-\varepsilon$ over $S\sim\DataDistri^{\otimes N}$, simultaneously for every $\rho\in\mathcal{P}(\Hypo)$ satisfying $\rho\ll\pi$,
\begin{equation}
\label{eq:catoni}
    \langle R \rangle_\rho \leq \Phi^{-1}_{\lambda/N}\!\left(
        \langle r \rangle_\rho
        + \frac{\mathrm{KL}(\rho \,\|\, \pi) + \log \varepsilon^{-1}}{\lambda}
    \right),
\end{equation}
where $\Phi^{-1}_{a}(x)=(1-\exp(-a\,x))/(1-\exp(-a))$ is the inverse Catoni
function and KL the Kullback--Leibler divergence.
Also, the bound applies to every posterior selected by a Markov
kernel, including $\rho_S=\learner(S,\cdot)$.
\end{theorem}

For $\lambda>0$, the \emph{Gibbs learning rule} returns
the posterior $\rho_{\lambda}$ defined by, for a realization $s$,
\begin{equation}\label{eq:gibbs}
\frac{\dif\rho_{\lambda}}{\dif\pi}(h) = \frac{e^{-\lambda r_s(h)}}{Z_\lambda(s)},\quad Z_\lambda(s)=\int e^{-\lambda r_s(h)}\pi(\dif h).
\end{equation}
Joint measurability makes the Gibbs learner $\learner_\lambda(s,B) \coloneqq \rho_{\lambda}(B)$, $B\in\HypoBorel$, a
Markov kernel. Importantly, $\rho_\lambda$ uniquely minimizes the
right-hand side of~\cref{eq:catoni}.

\begin{lemma}[Gibbs posterior minimizes the Catoni bound]
\label{lem:gibbs_min}
Fix $\lambda > 0$ and $\varepsilon \in (0,1)$, and define the Catoni certificate
\begin{equation}
    B_\lambda(\rho) \;\coloneqq\; \Phi^{-1}_{\lambda/N}\!\left(
        \langle r \rangle_\rho
        + \frac{\mathrm{KL}(\rho \,\|\, \pi) + \log \tfrac1\varepsilon}{\lambda}
    \right).
\end{equation}
Then, among all posteriors $\rho \ll \pi$, the Gibbs posterior $\rho_\lambda$
of \cref{eq:gibbs} is the unique minimizer of $B_\lambda$, and $B_\lambda(\rho_\lambda)=\Phi^{-1}_{\lambda/N}\!\left(\frac{-\log Z_\lambda + \log \varepsilon^{-1}}{\lambda}\right)$.
\end{lemma}
\begin{proof}
See~\cref{proof:lem:gibbs_min}.
\end{proof}

\begin{remark}[Gibbs identities]\label{rem:gibbs-identities}
Two identities used repeatedly below follow directly from $\dif\rho_\lambda/\dif\pi=e^{-\lambda r}/Z_\lambda$:
\begin{equation}\label{eq:gibbs-identity}
\KL(\rho_\lambda\Vert\pi) = -\lambda\langle r\rangle_{\rho_\lambda}-\log Z_\lambda,
\quad
\lambda\langle r\rangle_{\rho_\lambda}+\KL(\rho_\lambda\Vert\pi) = -\log Z_\lambda.
\end{equation}
The second identity is what makes the log-partition function a computable stand-in for the (empirical-risk $+$ KL) budget of a Gibbs posterior.
\end{remark}

\section{PAC-Bayes Analysis of Privileged Information}
\label{sec:theory}
Our analysis follows a potential--realization--transfer storyline. We first introduce two idealized endpoint Gibbs measures: the always-PI endpoint, which retains privileged information at prediction time, and the never-PI endpoint, which never observes it; their separation quantifies the potential of PI. Since PI is unavailable at inference, we then deploy the always-PI measure to the hypothesis space and ask how much of this potential can be realized. Finally, we transfer the theory from idealized Gibbs measures to a broad fixed-feature class of posteriors via a data-independent change of prior. \Cref{thm:main,thm:true_risk_separation} address the two levels of the argument: comparison of PAC-Bayes certificates and separation of the true endpoint risks, respectively.

\subsection{Basic assumptions}\label{sec:deployment_mape}

Let $(\Hypo^*,d_{\Hypo^*})$ be a complete separable metric space of
Borel-measurable PI hypotheses $h^*\colon\X\times\X^*\longrightarrow\R$. We define the Borel $\sigma$-algebra $\HypoBorel^*$ on $\Hypo^*$ and collection of measures $\mathcal M_1^+(\Hypo^*)$ in the same way as $(\Hypo, \HypoBorel)$. For a realization $s^*=((x_i,x_i^*,y_i))_{i=1}^N$,
we define
\begin{align}
R^*(h^*)&\coloneqq \E_{(X,X^*,Y)\sim\DataDistri^*}\!\left[\ell\bigl(h^*(X,X^*),Y\bigr)\right],\label{eq:pi-risk}\\
r_{s^*}^*(h^*)&\coloneqq \frac1N\sum_{i=1}^N\ell\bigl(h^*(x_i,x_i^*),y_i\bigr). \label{eq:pi-empirical-risk}
\end{align}
Accordingly, $r_{S^*}^*$ is a random functional and $r_{s^*}^*$ is its realization. For simplicity, we write $r^*$, and likewise $r$, whenever the relevant sample is unambiguous.

Since $x^*$ is unavailable at test time, a learned hypothesis in
$\mathcal{H}^*$ cannot be deployed directly. We therefore assume a Borel-measurable \emph{deployment map}
\begin{equation}
\label{eq:deploy}
    T\colon \Hypo^* \to \Hypo, \quad h^* \mapsto h,
\end{equation}
which is part of the learning architecture and independent of the training sample. For each $\rho^*\in\mathcal{M}_1^+(\Hypo^*)$, we denote by $[T_\sharp\rho^*](A)\coloneqq\rho^*(T^{-1}(A))$, $A\in\HypoBorel$, the pushforward measure. We also assume the linear-growth condition $d_{\Hypo}(T(h^*),h_0)
\leq C_T\bigl(1+d_{\Hypo^*}(h^*,h_0^*)\bigr)$, $h_0 = T(h_0^*)$, for some $C_T>0$ (e.g., $T$ is Lipschitz-continuous), which guarantees that $T_\sharp\rho^*\in\mathcal{M}_1^+(\Hypo)$ whenever $\rho^*\in\mathcal{M}_1^+(\Hypo^*)$.
Let $\pi^*\in\mathcal M_1^+(\Hypo^*)$ be a data-independent prior. To compare
the PI and non-PI constructions under the same prior belief, we take $\pi\coloneqq T_\sharp\pi^*$; since $T$ and $\pi^*$ are data-independent, $\pi$ is a valid PAC-Bayes prior for the Catoni bound in~\cref{thm:catoni}.

\begin{example}[Marginalization construction for $T$]\label{ex:marginalization}
Assume $\X,\X^*$ are standard Borel and let $\mu\colon \X\times\HypoBorel^*\to[0,1]$ be a Markov kernel. Denote by $\mu_x \coloneqq \mu(x, \cdot)$ a measure on $(\Hypo^*, \HypoBorel^*)$. A canonical choice is a regular conditional probability distribution of $X^*$ given $X=x$, i.e., for each $B\in\HypoBorel^*$, $\mu_x(B) = \Pr(X^*\in B \mid X = x)$; under \cref{ass:pi-data} this kernel is independent of the learner. We define $T\colon\Hypo^*\to\Hypo$ by
$$
[Th^*](x)=\int_{\Hypo^*} h^*(x,x^*)\mu_x(\dif x^*)\,.
$$
If this integral is absolutely finite for every $x$ and the result belongs to $\Hypo$, kernel integration
makes $(h^*,x)\mapsto(Th^*)(x)$ jointly measurable. Note that, we assume the conditional probability distribution $\mu_x$ is independent from the sample $S$ or $S^*$; otherwise, the pushforward of the prior $\pi\coloneqq T_\sharp\pi^*$ will be sample-dependent, which fails the requirements of the Catoni bound.
\end{example}

\subsection{Endpoint posteriors and the Gibbs envelope} \label{sec:envelope}
Fix the inverse temperature $\lambda$, a realization $s^*$ and its projection $s$. The always-PI and never-PI Gibbs posteriors are defined by
$$
\frac{\dif\rhoAlways}{\dif\pi^*}(h^*) \coloneqq\frac{e^{-\lambda r^*(h^*)}}{Z_\lambda^*(s^*)},\quad  \frac{\dif\rhoNever}{\dif\pi}(h) \coloneqq\frac{e^{-\lambda r(h)}}{Z_\lambda(s)},
$$
where $Z_\lambda^*(s^*) \coloneqq \E_{h^*\sim\pi^*}\exp(-\lambda r^*(h^*))$ and $Z_\lambda(s) \coloneqq \E_{h\sim\pi}\exp(-\lambda r(h))$. We refer to $\rhoAlways$ and $\rhoNever$ jointly as the \emph{endpoint
measures}. The never-PI endpoint is the ordinary baseline. The always-PI
endpoint is an oracle benchmark: it is allowed to retain $x^*$ at prediction
time and, for fixed $(\lambda,\pi^*)$, uniquely minimizes the PI Catoni
certificate over all $\rho^*\ll\pi^*$. The always-PI $\rhoAlways\in \mathcal{M}^+_1(\Hypo^*)$ and its true risk should be measured with functional $R^*$. In this precise PAC-Bayes sense, it is
the best Gibbs posterior available with PI. It is not, however, a deployable
LUPI posterior.

The deployment map $T$ pushes forward the oracle endpoint $\rhoAlways$ to $\Hypo$, resulting in the LUPI posterior $\rhoDeployed\in\Measures(\Hypo)$,
\begin{equation*}
\rhoDeployed \coloneqq T_\sharp\rhoAlways, \quad \langle R\rangle_{\rhoDeployed}=\E_{h^*\sim\rhoAlways}\!\bigl[\,R\bigl(T(h^*)\bigr)\,\bigr].
\end{equation*}
The endpoint posteriors depend only on their respective empirical risks and priors, whereas $\rhoDeployed$ also depends on how the deployment map $T$ transfers PI to $\Hypo$. Deployment can lose (part of) the PI. For instance, nothing in our construction forbids a degenerate choice of deployment, i.e., $T(h^*)\equiv h_0$ for a fixed $h_0$, which discards the hypothesis learned with PI entirely. We call the signed endpoint gap
\begin{equation} \label{eq:gibbs-envelope}
\mathcal{G}_\lambda \coloneqq \langle R\rangle_{\rhoNever} -\langle R^*\rangle_{\rhoAlways}
\end{equation}
the \emph{Gibbs envelope}. A positive $\mathcal{G}_\lambda$ means that PI opens
a genuine gap between the two idealized endpoints. If, in addition, deployment
places $\rhoDeployed$ between them, i.e.,
\begin{equation}
\label{eq:envelope}
    \underbrace{\langle R^*\rangle_{\rhoAlways}}_{\text{best: keep }x^*}
    \le
    \underbrace{\langle R\rangle_{\rhoDeployed}}_{\text{LUPI: deploy via }T}
    \le
    \underbrace{\langle R\rangle_{\rhoNever}}_{\text{baseline: never use }x^*}
\end{equation}
then $\mathcal{G}_\lambda$ upper-bounds the realized deployment gain, and a wide envelope is \emph{necessary for LUPI to be beneficial but not sufficient}. We do not
assume~\cref{eq:envelope} to always hold: the endpoint separation and the deployment gain are
certified separately in the next subsection.
This distinction separates two questions: (1) \textbf{Potential}, whether PI creates a positive Gibbs envelope, and
(2) \textbf{Realization}, whether the chosen $T$ converts that potential into an LUPI posterior better than the baseline $\rhoNever$.

\subsection{Partition-function criteria for PI utility}
\label{sec:metric}
To answer the \textbf{Potential} question, we first ask whether PI improves the optimal \emph{upper certificate} at the oracle endpoint. Let $B_\lambda^*(\rho^*)$ denote the PI analogue of
$B_\lambda(\rho)$ in~\cref{lem:gibbs_min}, formed with $(r^*,\pi^*)$. We use
the abbreviations $B_\lambda^*\coloneqq B_\lambda^*(\rhoAlways)$ and $B_\lambda\coloneqq B_\lambda(\rhoNever)$. For the observed pair $(s^*,s)$, we define the log-partition gap
\begin{equation} \label{delta_z}
\Delta Z(s^*)\coloneqq \log Z_\lambda^*(s^*)-\log Z_\lambda(s).
\end{equation}
Throughout, $\log$ denotes the natural logarithm. Because the endpoint posteriors minimize their respective empirical PAC-Bayes
objectives, this gap can be evaluated from prior integrals before either
posterior is normalized. The next theorem shows that the sign of $\Delta Z$ orders the two endpoint upper certificates.

For a Gibbs measure, the KL term and empirical risk combine into the log-partition function (\cref{rem:gibbs-identities}), an identity that makes the comparison both exact and computable; its extension beyond the canonical Gibbs form is discussed in
\cref{sec:expfam}.

\begin{theorem}[Criterion for the tighter endpoint certificate]
\label{thm:main}
Fix $\lambda>0$ and $\varepsilon\in(0,1/2)$. With probability at least
$1-2\varepsilon$ over $S^*\sim(\DataDistri^*)^{\otimes N}$ and its projection
$S$, both endpoints' upper certificates are valid:
\begin{equation} \label{eq:endpoint-upper-certificates}
\langle R^*\rangle_{\rhoAlways}\le B_\lambda^*, \quad \langle R\rangle_{\rhoNever}\le B_\lambda.
\end{equation}
Conditionally on the observed sample, the following equivalences hold:
\begin{equation} \label{eq:main-equiv}
B^*_\lambda < B_\lambda\Longleftrightarrow \Delta_r > \frac{\Delta_{\mathrm{KL}}}{\lambda} \Longleftrightarrow \Delta Z > 0,
\end{equation}
where $\Delta_{\mathrm{KL}} = \mathrm{KL}(\rhoAlways \| \pi^*) - \mathrm{KL}(\rhoNever\| \pi)$ and $\Delta_r = \langle r \rangle_{\rhoNever} - \langle r^* \rangle_{\rhoAlways}$.
\end{theorem}
\begin{proof}
See~\cref{proof:thm:main}.
\end{proof}

\paragraph{Proof idea.} Both certificates come from Catoni's bound plus a union bound. Since $\Phi^{-1}_{\lambda/N}$ is strictly increasing, $B^*_\lambda<B_\lambda$ is equivalent to an ordering of their arguments, in which the confidence terms $\log\varepsilon^{-1}/\lambda$ cancel and leave $\Delta_r>\Delta_{\mathrm{KL}}/\lambda$. Substituting the Gibbs identities~\cref{eq:gibbs-identity} collapses this quantity exactly to $\Delta Z/\lambda$, so the whole condition reduces to $\Delta Z>0$. Intuitively, $\Delta_r$ measures how much PI lowers the fitted (empirical) risk, while $\Delta_{\mathrm{KL}}/\lambda$ is the extra model-complexity price paid on the richer PI space; PI helps the certificate exactly when the former outweighs the latter.

\begin{remark}
$\Delta Z$ is, therefore, a pre-training \emph{criterion}, rather than a risk
certificate itself. It can be estimated from the empirical risks and prior
integrals without fitting a normalized endpoint posterior. A positive value
certifies that the always-PI endpoint has the tighter upper certificate; it
does not yet imply that the Gibbs envelope in~\cref{eq:gibbs-envelope} is positive.
\end{remark}

\Cref{thm:main} isolates the potential of PI visible at the oracle endpoint.
To study the \textbf{Realization} question, we must account for what changes when that endpoint measure is transported through $T$.
We define the empirical deployment gap and the
relative-entropy contraction by
\begin{align}
D_T(s^*) &\coloneqq \langle r\circ T-r^*\rangle_{\rhoAlways} =\langle r\rangle_{\rhoDeployed}-\langle r^*\rangle_{\rhoAlways},\label{eq:deployment-gap} \\
\kappa_T(s^*) &\coloneqq \KL(\rhoAlways\Vert\pi^*) -\KL(\rhoDeployed\Vert\pi)\,. \label{eq:kl-contraction}
\end{align}
Here $(r\circ T)(h^*)=r(T(h^*))$. We present an upper certificate for the LUPI true risk as follows.

\begin{corollary}[Deployment-certificate criterion] \label{corollary:deployment-certificate}
For any fixed $\lambda>0$ and $\varepsilon\in(0,1)$, with probability at least
$1-\varepsilon$ over $S^*\sim(\DataDistri^*)^{\otimes N}$,
\begin{equation} \label{eq:deployed-risk-certificates}
\langle R\rangle_{\rhoDeployed}\leq B_\lambda(\rhoDeployed), \quad \langle R\rangle_{\rhoNever}\leq B_\lambda(\rhoNever).
\end{equation}
Conditionally on the observed sample $s^*$, we have
\begin{equation*}
B_\lambda(\rhoDeployed)<B_\lambda(\rhoNever) \Longleftrightarrow \Delta Z(s^*)>\lambda D_T(s^*)-\kappa_T(s^*).
\end{equation*}
Moreover, $\kappa_T(s^*)\geq 0$. Hence, the posterior-free condition
\begin{equation} \label{eq:sufficient-deployment}
\Delta Z(s^*) > \lambda D_T(s^*)
\end{equation}
is sufficient for the LUPI upper certificate to be tighter than the
never-PI certificate. If $\kappa_T(s^*)=0$, it is also necessary.
\end{corollary}
\begin{proof}
See~\cref{proof:corollary:deployment-certificate}.
\end{proof}

\paragraph{Proof idea.} Applying Catoni on the augmented sample space to the $x$-only loss $\ell(h(x),y)$ covers both $\rhoDeployed$ and $\rhoNever$ simultaneously, so no union bound is needed. Comparing the two certificate arguments and inserting the Gibbs identities together with the definitions of $D_T$ and $\kappa_T$ gives the stated equivalence. The sign $\kappa_T\ge0$ is the data-processing inequality for KL under the pushforward $T$: transporting both $\rhoAlways$ and $\pi^*$ through the same $T$ can only contract their divergence.

\begin{remark}
The sufficient criterion~\cref{eq:sufficient-deployment} does not require a
separate posterior-fitting procedure. We have
$
D_T(s^*) = Z_\lambda^*(s^*)^{-1}\int_{\Hypo^*}[r(T(h^*))-r^*(h^*)] \exp(-\lambda r^*(h^*))\pi^*(\dif h^*).
$
Thus $\Delta Z$ and $D_T$ can be estimated from samples $h^*\sim\pi^*$ whenever $T(h^*)$, $r(T(h^*))$, and $r^*(h^*)$ are accessible.
\end{remark}

Both~\cref{thm:main,corollary:deployment-certificate} compare upper
certificates, not the true risks themselves. In particular,
$B_\lambda(\rhoDeployed)<B_\lambda(\rhoNever)$ does not imply
$\langle R\rangle_{\rhoDeployed}<\langle R\rangle_{\rhoNever}$. A true-risk
separation requires an upper certificate for one posterior and a lower
certificate for the other. The latter is obtained by applying Catoni's bound
to the complementary loss.

\begin{lemma}[Catoni lower certificate] \label{lem:lower_cert}
Let $\pi$ be a data-independent prior over $\mathcal{H}$, $\ell$ a bounded
loss, and $\lambda > 0$. For every posterior $\rho$, with probability at
least $1 - \varepsilon$ over the draw of $S$,
\begin{equation*}
    \langle R\rangle_\rho \geq L_\lambda(\rho) \coloneqq 1 \!-\! \Phi^{-1}_{\lambda/ N}\!\left(
        1 \!-\! \langle r\rangle_\rho \!+\! \frac{\KL(\rho \| \pi) \!+\! \ln\frac{1}{\varepsilon}}{\lambda}
    \right)\,.
\end{equation*}
\end{lemma}
\begin{proof}
See~\cref{proof:lem:lower_cert}.
\end{proof}

\paragraph{Proof idea.} Apply the Catoni upper bound to the complementary loss $\ell'=1-\ell\in[0,1]$, whose posterior-averaged true and empirical risks are $1-\langle R\rangle_\rho$ and $1-\langle r\rangle_\rho$; rearranging turns the upper bound on $1-\langle R\rangle_\rho$ into a lower bound on $\langle R\rangle_\rho$.

The lower certificate turns the potential question into a two-sided test. The next result gives an exact criterion for separating the certificates and a sufficient criterion for proving that the Gibbs envelope is positive.

\begin{theorem}[Two-sided certificate for a positive Gibbs envelope] \label{thm:true_risk_separation}
Fix $\lambda>0$ and $\varepsilon\in(0,1/2)$.
For the always-PI and never-PI Gibbs posteriors, define $a^*\coloneqq \langle r^*\rangle_{\rhoAlways}
+(\KL(\rhoAlways\Vert\pi^*)+\log\varepsilon^{-1})/\lambda$ and $b\coloneqq 1-\langle r\rangle_{\rhoNever}
+(\KL(\rhoNever\Vert\pi)+\log\varepsilon^{-1})/\lambda$.
With probability at least $1-2\varepsilon$ over
$S^*\sim(\DataDistri^*)^{\otimes N}$ and its projection $S$,
\begin{equation} \label{eq:endpoint-two-sided-certificates}
\langle R^*\rangle_{\rhoAlways}\leq B_\lambda^*, \quad \langle R\rangle_{\rhoNever}\geq L_\lambda(\rhoNever).
\end{equation}
Conditionally on a realization $s^*$, separation of the two certificates
is characterized exactly by
\begin{equation} \label{eq:endpoint-separation}
B_\lambda^* < L_\lambda(\rhoNever) \Longleftrightarrow e^{-a^*\frac\lambda N} + e^{-b\frac\lambda N} > 1 + e^{-\frac\lambda N}.
\end{equation}
On the event in~\cref{eq:endpoint-two-sided-certificates}, either side of~\cref{eq:endpoint-separation} implies
\begin{equation} \label{eq:endpoint-true-risk-separation}
\langle R^*\rangle_{\rhoAlways} < \langle R\rangle_{\rhoNever}.
\end{equation}
Moreover, the following condition is sufficient for~\cref{eq:endpoint-separation}:
\begin{align}
\Delta_r &\coloneqq \langle r\rangle_{\rhoNever} - \langle r^*\rangle_{\rhoAlways} >\frac{\KL(\rhoAlways\Vert\pi^*)+\KL(\rhoNever\Vert\pi)+2\ln\frac1\varepsilon}{\lambda} + C_\lambda \label{eq:endpoint-separation-sufficient}
\end{align}
where $C_\lambda = 1-\frac{N(1-e^{-\lambda/N})}{\lambda}$. Since $\rhoAlways,\rhoNever$ are Gibbs measures, the above condition is equivalent to
\begin{equation} \label{eq:endpoint-partition-separation}
\log\left(Z_\lambda Z_\lambda^*\right) > 2\!\left(\ln\frac1\varepsilon - \lambda\langle r\rangle_{\rhoNever}\right) + \lambda C_\lambda.
\end{equation}
\end{theorem}
\begin{proof}
See~\cref{proof:thm:true_risk_separation}.
\end{proof}

\paragraph{Proof idea.} Pair the upper certificate $B_\lambda^*$ for the always-PI endpoint (\cref{thm:main}) with the lower certificate $L_\lambda(\rhoNever)$ (\cref{lem:lower_cert}) via a union bound. Writing both in terms of $\Phi^{-1}_{\lambda/N}$ and substituting its closed form turns $B_\lambda^*<L_\lambda(\rhoNever)$ into the exponential inequality~\cref{eq:endpoint-separation}. The linear upper bound $\Phi^{-1}_c(x)\le cx/(1-e^{-c})$ yields the simpler sufficient condition on $\Delta_r$, which the Gibbs identities rewrite as the product-of-partitions condition~\cref{eq:endpoint-partition-separation}. Note this controls the \emph{product} $Z_\lambda Z_\lambda^*$: $\Delta Z>0$ (i.e.\ $Z_\lambda^*>Z_\lambda$) is necessary but not sufficient for a true-risk separation.

\begin{corollary}[Deployment true-risk separation] \label{corollary:deployment-risk-separation}
In the setting of~\cref{thm:true_risk_separation}, define $u \coloneqq \langle r\rangle_{\rhoDeployed} + \frac{\KL(\rhoDeployed\Vert\pi)+\log\varepsilon^{-1}}{\lambda}$.
With probability at least $1-2\varepsilon$ over $S^*\sim(\DataDistri^*)^{\otimes N}$,
\begin{equation}
\langle R\rangle_{\rhoDeployed}\leq B_\lambda(\rhoDeployed), \quad \langle R\rangle_{\rhoNever}\geq L_\lambda(\rhoNever).
\end{equation}
Conditionally on $s^*$, we have
\begin{equation*}
B_\lambda(\rhoDeployed) < L_\lambda(\rhoNever) \Longleftrightarrow
e^{-u\frac\lambda N}+e^{-b\frac\lambda N}>1+e^{-\frac\lambda N}.
\end{equation*}
A sufficient condition is
\begin{align}
&\langle r\rangle_{\rhoNever}-\langle r\rangle_{\rhoDeployed} > \frac{\KL(\rhoDeployed\Vert\pi) + \KL(\rhoNever\Vert\pi) + 2\ln\frac1\varepsilon}{\lambda} + C_\lambda\,. \label{eq:deployment-separation-sufficient}
\end{align}
\end{corollary}
\begin{proof}
See~\cref{proof:corollary:deployment-risk-separation}.
\end{proof}

\subsection{From endpoint Gibbs posteriors to practical algorithms}
\label{sec:expfam}
The preceding results are stated for the canonical Gibbs form
$\rho\propto\pi e^{-\lambda r}$, which has a limited scope. We now extend them to the fixed-feature subclass of the exponential family~\citep{wainwright2008graphical}. We identify an exact extension: any data-independent term in the log-density of this class can be absorbed into a reshaped prior, resulting in an ordinary Gibbs measure on the empirical risk, to which we can directly transfer our theory of~\cref{sec:metric}.

\begin{definition}[Fixed-feature subclass of exponential family]
\label{def:fixed_feature}
A posterior $\rho\ll\pi$ is called \emph{fixed-feature} if there exist a
data-independent measurable map $\widetilde\phi\colon\Hypo\to\R^k$, a fixed vector $\mu\in\R^k$, and $\lambda>0$ such that
\begin{equation}
\label{eq:fixed_feature}
    \log\frac{\dif\rho}{\dif\pi}(h)
    = -\lambda\, r(h) - \langle\mu,\tilde\phi(h)\rangle - \psi(\lambda,\mu),
\end{equation}
where $\psi(\lambda,\mu)$ is the log-normalizing constant. Thus, the random sample $S$ enters the density only through $r$. The feature $\widetilde\phi$ and its coefficient $\mu$ describe a fixed shape.
\end{definition}

This class contains exact Bayesian posteriors for linear, generalized-linear,
and Gaussian-process models. Quadratic Gibbs members have means and modes that
recover kernel ridge, frozen last-layer, and NTK-linearized estimators, and
include Gaussian variational laws only under an explicit fixed-natural-parameter
condition. Each claim and its limitations are verified separately
in~\cref{sec:fixed-feature-exp-family-example}.

\begin{lemma}[Reduction to a Gibbs posterior]\label{lem:reduction}
Define the reshaped prior $\pi_\mu \propto \pi\, e^{-\langle\mu,\tilde\phi\rangle}$. Every fixed-feature posterior is the Gibbs posterior on $r$ at temperature $\lambda$ relative to $\pi_\mu$:
\begin{equation} \label{eq:fixed-feature-gibbs}
\frac{\dif\rho}{\dif\pi_\mu}(h) =\frac{e^{-\lambda r(h)}}{\E_{h\sim\pi_\mu}e^{-\lambda r(h)}}\,.
\end{equation}
\end{lemma}
\begin{proof}
See~\cref{proof:lem:reduction}.
\end{proof}

The reduction is exact, but its data-independence condition is essential. If
$(\widetilde\phi,\mu)$ is fixed before observing the sample, then $\pi_\mu$ is
a valid PAC-Bayes prior, and the endpoint results
in~\cref{thm:main,thm:true_risk_separation} apply with
$\pi$ replaced by $\pi_\mu$ and $Z_\lambda$ by $\mathbb{E}_{\pi_\mu}[e^{-\lambda r}]$, with the analogous substitutions on $\Hypo^*$.

\section{Empirical Validation} \label{sec:verification}
We evaluate the log-partition function gap~\cref{delta_z} between two endpoint posteriors motivated by~\cref{thm:main},
i.e., $\Delta Z>0$, and whether its magnitude tracks the realized gain from PI.
In both studies, $T$ is identified with an identity map on the parameter space shared by \Hypo and $\Hypo^*$, hence eliminating the need to consider the additional deployment distortion of PI. Non-identity examples of $T$ are given in~\cref{sec:nontrivial-deployment-examples}.

\subsection{Bayesian estimation of a Gaussian mixture} \label{sec:GMM}
Our first case study is Bayesian parameter estimation for a two-component
isotropic Gaussian mixture model (GMM):
$$
p_\theta(x)=w\,\mathcal{N}(x;\mu_1,\sigma_1^2 I_n)+(1-w)\,\mathcal{N}(x;\mu_2,\sigma_2^2 I_n)\,,
$$
where $\mathcal{N}(x;\mu,\sigma^2 I_n)$ denotes the probability density at point $x\in\R^n$ of a Gaussian r.v.\ with parameter $(\mu, \sigma^2)$. Denoting by $\Theta\coloneqq [0,1]\times\R^n\times\R^n\times\R_{>0}\times\R_{>0}$ the parameter space of GMM, the hypothesis space thereof is
$$
\Hypo_{\textsc{gmm}} = \{p_\theta \mid \theta=(w,\mu_1,\mu_2,\sigma_1^2,\sigma_2^2)\in\Theta\}\,.
$$
In this study, we set the privileged information to be, for each $x_i$, the label $z_i\in\{1,2\}$ of the Gaussian cluster from which $x_i$ is sampled.
Therefore, the hypothesis of the LUPI case is the joint density function of $(x,z)$,
$$
p_\theta^*(x,z)= w_z\mathcal N(x;\mu_z,\sigma_z^2I_n),\quad \Hypo^*_{\textsc{gmm}}=\{p_\theta^* \mid \theta\in\Theta\}\,,
$$
where we let $w_1=w$ and $w_2=1-w$.
Thus, both hypothesis spaces are parameterized/identified by the same parameter space $\Theta$. However, $\theta\mapsto p_\theta$ is not a bijective identification, e.g., swapping the label of Gaussian clusters leads to the same $p_\theta$. We consequently formulate Bayesian inference on $\Theta$ and understand the corresponding measures on $\Hypo_{\textsc{gmm}}$ and $\Hypo_{\textsc{gmm}}^*$ as pushforwards under $\theta\mapsto p_\theta$ and $\theta\mapsto p_\theta^*$, respectively. Deployment marginalizes the training-only label $z$,
\begin{equation} \label{eq:gmm-deployment}
[Tp_\theta^*](x) \coloneqq p_\theta^*(x, 1) + p_\theta^*(x, 2) = p_\theta(x)\,.
\end{equation}
Equivalently, $T$ can be identified with the identity map on $\Theta$. We, therefore, place one proper, data-independent prior $\pi$ on $\Theta$ and use its pushforwards as the priors on both hypothesis spaces. We take a Beta prior for $w$ and independent Normal--Inverse-Gamma (NIG) priors for $(\mu_z,\sigma_z^2)$, $z\in\{1,2\}$; see~\cref{sec:gmm-details} for details.
Given a PI-augmented data point $s^*=((x_i, z_i))_{i=1}^N$ and its projection $s$, we took the negative log-likelihood for the empirical risk functional:
$$
r(\theta) = -\frac1N\sum_{i=1}^N\log p_\theta(x_i), \quad r^*(\theta) = -\frac1N\sum_{i=1}^N\log p_\theta^*(x_i, z_i)\,.
$$
We note that the risks are unbounded, so the Catoni certificates do not apply literally; therefore, we use $\Delta Z$ diagnostically and verify it tracks the realized gain, leaving a bounded-surrogate or unbounded-loss guarantee to future work. Because these are average negative log-likelihoods, we have, for $\lambda=N$, the Gibbs potential satisfies $\exp(-\lambda r(\theta)) =\prod_{i=1}^N p_\theta(x_i)$ and $\exp(-\lambda r^*(\theta)) = \prod_{i=1}^N p_\theta^*(x_i,z_i)$.
Hence, the Bayesian posteriors coincide exactly with the two endpoint Gibbs measures, e.g., for $\rhoNever$ and $B\in \mathcal{B}(\Theta)$,
\begin{align}
\Pr(\theta \in B \mid S=s) &= \frac{\int_B p_\theta(s) \pi(\dif\theta)}{\int_\Theta p_{\theta'}(s)\pi(\dif\theta')} = \int_B\frac{e^{-\lambda r(\theta)}}{Z_\lambda(s)}\pi(\dif\theta) = \rhoNever(B)\,,
\end{align}
where $Z_\lambda(s)$ is precisely the Bayesian marginal likelihood.
Here, no prior reshaping through~\cref{lem:reduction} is required. For $\lambda\neq N$, the factor $\exp(-\lambda r(\theta))=\prod_i p_\theta(x_i)^{\lambda/N}$ instead defines a tempered posterior; we therefore take $\lambda=N$ throughout this example. Both endpoint Gibbs measures are computable.

Without PI, the likelihood is a product of mixtures and is not conjugate to
the Beta--NIG prior; we sample $\rhoNever$ using a Metropolis--Hastings algorithm. The PI likelihood then factorizes into a Beta--Binomial term and two NIG terms, allowing exact ancestral sampling for $\rhoAlways$. Taking samples from both endpoint measures, $\Delta Z$ can be estimated with Monte Carlo integration. Full posterior updates and computational details are given in \cref{sec:gmm-details}.

\begin{figure}[t]
    \centering
    \includegraphics[width=0.85\linewidth,trim=02mm 05mm 02mm 02mm,clip]{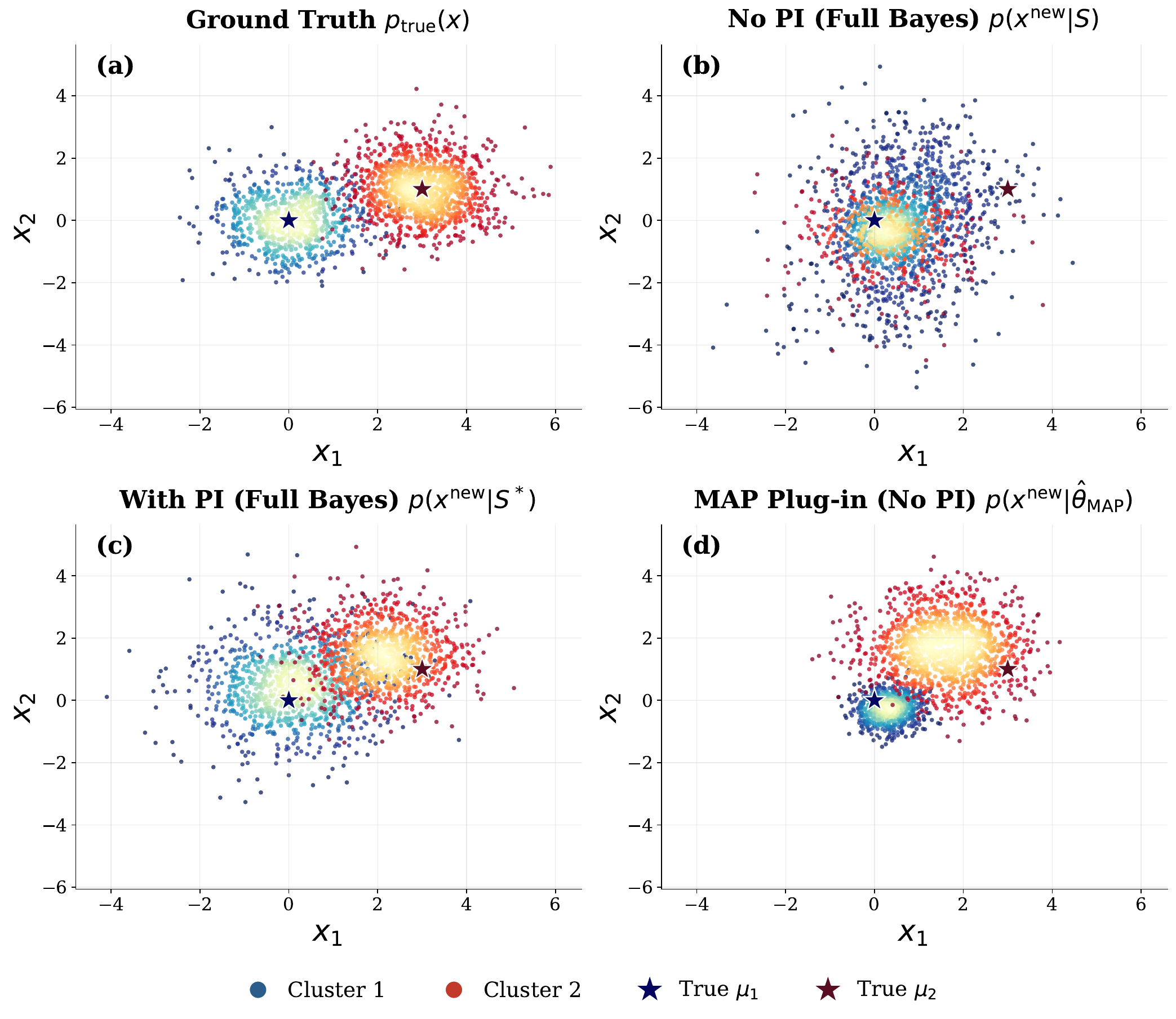}
    \caption{Posterior predictive densities for a two-component GMM
    trained on $N=8$ observations. (a) the generating
    density, (b) the predictive distribution with $\rhoNever$, (c) the predictive distribution with $\rhoDeployed$, (d) the MAP plug-in density.}
    \label{fig:gmm-predictive}
\end{figure}

\begin{figure}[t]
    \centering
    \includegraphics[width=0.65\linewidth,trim=02mm 02mm 02mm 02mm,clip]{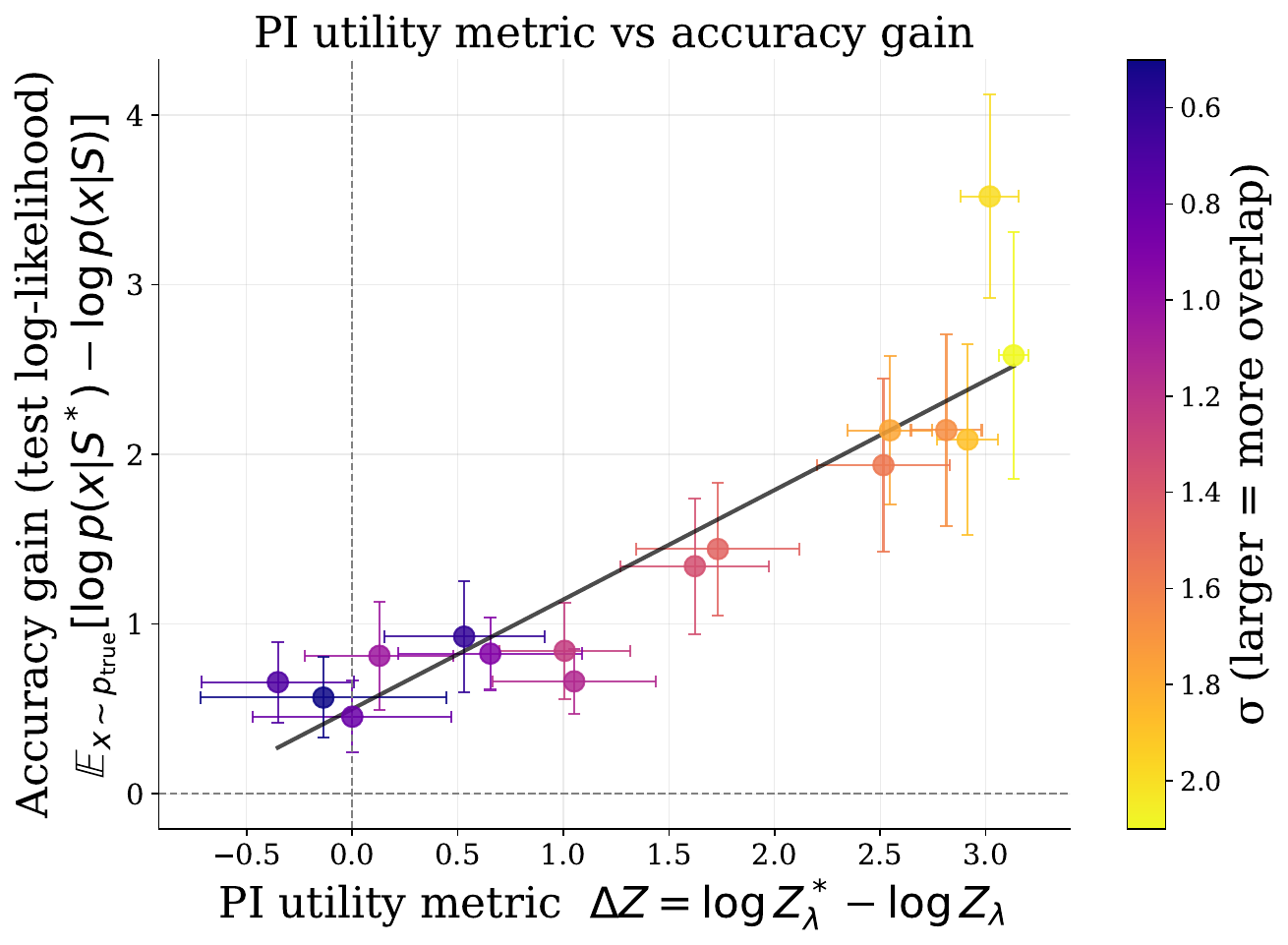}
    \caption{PI utility metric versus realized gain. Each point is one
    cluster-overlap setting $\sigma$ (colour), averaged over $14$ random
    datasets of $N=8$ points. As overlap decreases, both $\Delta Z$ and the
    test log-likelihood gain from privileged information grow together
    ($r=0.911$).}
    \label{fig:correlation}
\end{figure}

At inference time, we compare three predictors: (1) Bayesian predictive distribution w/o PI (No PI (Full Bayes)), i.e., $\overline{p}_{\text{never}}(x)\coloneqq
\int_\Theta p_\theta(x) \rhoNever(\dif\theta)$; (2) Bayesian predictive distribution w/ PI (With PI (Full Bayes)), i.e., $\overline{p}_{\textsc{lupi}}(x)\coloneqq
\int_\Theta p_\theta(x) \rhoAlways(\dif\theta)$; and (3) maximum a posteriori estimation (MAP) w/o PI (MAP Plug-in (No PI)), included as a point-estimation baseline.

\Cref{fig:gmm-predictive} illustrates these predictives for a very low number of training samples $N=8$. The
privileged labels reduce component ambiguity and yield a posterior predictive
that is more concentrated around the generating components.
In~\cref{fig:correlation}, we vary the component overlap and average over
$14$ datasets per setting.
The strong positive association between $\Delta Z$ and the testing accuracy gain (Pearson $r=0.911$) shows that the partition gap
successfully predicts the realized gain from PI across overlap regimes.

\subsection{Gaussian Process Classification with Privileged Noise (GPC+)}
We extend our analysis to Gaussian Process Classification (GPC)~\citep{McKay} as a nonparametric Bayesian method for binary classification.
We study whether the partition-function gap $\Delta Z$ can retain its diagnostic power in a supervised scenario as well.

Assume $k\colon \X\times\X \to \R$, a real-valued positive definite function (a.k.a.\ kernel). Let $\mathcal{H}_k$ denote the reproducing kernel Hilbert space (RKHS) induced by kernel $k$. Consider the Banach space $C(\X)$ consisting of continuous latent functions $f\colon\X\to\R$, equipped with uniform norm $\lVert\cdot\rVert_\infty$.
We endow a data-independent prior $\pi=\mathcal{GP}(0,k)$ on $C(\X)$, meaning that for any finite set $X = \{x_i\}_{i=1}^N \subset \X$, the prior $\pi$ stipulates $\mathbf{f} \coloneqq (f(x_1), \ldots, f(x_N))\sim \mathcal{N}(0, K)$ with $K\in S^N_{++}$, the cone of positive definite matrices, and $K_{ij} = k(x_i, x_j)$. Then, the latent-function hypothesis space is
$
\Hypo_{\textsc{gpc}} \coloneqq \operatorname{supp}(\pi) = \overline{\mathcal{H}_k}^{\,\|\cdot\|_\infty} \subseteq C(\X).
$
For the true labels $\mathbf{y}=(y_1, \ldots, y_N)$ corresponding to $X$ with $y_i\in\{-1,1\}$, the model assumes a noisy threshold variable $z_i = f(x_i) + u_i$ with $u_i \sim \mathcal{N}(0, 1)$ and $y_i = \operatorname{sign}(z_i)$, resulting in the probit likelihood
\begin{equation*}
p(y_i \mid f(x_i)) = \Pr\big(y_i (f(x_i) + u_i) > 0\big) = \Phi(y_i f(x_i)),
\end{equation*}
where $\Phi(x) = \tfrac{1}{2}\big(1 + \operatorname{erf}(x/\sqrt{2})\big)$ is the
standard normal CDF. Inference amounts to computing the posterior
$p(\mathbf{f} \mid X, \mathbf{y}) \propto p(\mathbf{y} \mid \mathbf{f}) \pi(\mathbf{f})$, which --- unlike in the conjugate GMM case --- is
non-Gaussian and must be approximated.

PI indicates the reliability of each training label, and it enters the decision threshold via an extra noise term $\varepsilon_i \sim \mathcal{N}(0, \sigma^2_i)$ that is additive, independent, and localized~\citep{GPC+}. In our case, we use a pre-defined mapping for the privileged noise and the corresponding per-point standard deviation, whose role is shown in~\cref{pi_noise_probit}:
\begin{equation}
\label{eq:noise and variance}
\sigma^2_i\coloneqq \sigma_0 e^{-\lVert x^*_i\rVert/\tau}, \quad q(x^*_i)\coloneqq\sqrt{1+\sigma^2_i},
\end{equation}
where $x^*_i$ is the PI variable w.r.t.\ $x_i$, indicating the reliability of the label $y_i$. Combining the noise yields $u_i + \varepsilon_i \sim \mathcal{N}(0, 1 + \sigma^2_i)$, which expands the probit variance:
\begin{equation}
\label{pi_noise_probit}
p(y_i \mid f(x_i), x_i^*) =  \Phi\left(\frac{y_i f(x_i)}{q(x_i^*)}\right)\,.
\end{equation}
The LUPI hypothesis space (GPC+) is
$
h^*_f(x,x^*)\coloneqq f(x)/q(x^*),\;  \Hypo_{\textsc{gpc}}^*\coloneqq\{h_f^*\colon f\in\Hypo_{\textsc{gpc}}\}\,.
$
Both hypothesis spaces are therefore indexed by the same latent function
$f$. Since $q$ is fixed and strictly positive, this indexing is one-to-one, hence the deployment map is $T h_f^*\coloneqq f$;
so $T$ is represented by the identity map on the common $f$-parameterization. For a realization $s^*=((x_i,x_i^*,y_i))_{i=1}^N$ and its projection $s$, we took the following two empirical risk functionals:
$
r(f) =-\frac1N\sum_{i=1}^N\log\Phi\!\left(y_if(x_i)\right)$, $r^*(h^*_f) =-\frac1N\sum_{i=1}^N \log\Phi\!\left(\frac{y_if(x_i)}{q(x_i^*)}\right).
$
As in \cref{sec:GMM}, the probit log-loss is unbounded, so $\Delta Z$ is again used diagnostically. The PI scale flattens the likelihood of examples with $\lVert x_i^*\rVert$ close to
zero and hence reduces their influence. As in \cref{sec:GMM}, the Bayesian posterior coincides with the endpoint Gibbs measures precisely at $\lambda=N$, i.e., for any Borel set $B$ in $\Hypo$, $\rhoNever(B) = \Pr(\mathbf{f} \in B \mid X, \mathbf{y})$ and $\rhoAlways(B) = \Pr(\mathbf{f} \in B \mid X, X^*, \mathbf{y})$.
The partition functions are marginal likelihoods
\begin{align}
  Z_\lambda&=\int\!\prod_{i=1}^N\Phi(y_if(x_i))\pi(\dif\mathbf{f}),\\
  Z_\lambda^* &=\int\!\prod_{i=1}^N\Phi\!\left[\frac{y_if(x_i)}{q(x_i^*)}\right]\pi(\dif\mathbf{f})\,,
\end{align}
where $\pi(\dif\mathbf{f}) = \mathcal{N}(\mathbf{f};0, K)\dif\mathbf{f}$.
Neither integral is available in closed form. We approximate both posteriors
and both log evidences by expectation propagation (EP), using the same GP
prior and convergence rule in the two endpoints, and compute
$\Delta Z=\log Z_N^*-\log Z_N$. At test time, PI is absent and both models
predict from the deployed latent function $f$ using $x$ alone.

We evaluate this construction on a noisy circular-boundary task in
$\R^2$. With margin
$m_i=x_{i,1}^2+x_{i,2}^2-r_0^2$, the clean label is
$y_i^{\rm true}=\operatorname{sign}(m_i)$, while labels close to the boundary
are corrupted during training. The privileged feature is a noisy observation
$x_i^*=|m_i|+\xi_i^*$ of boundary proximity, where $\xi_i^*$ has noise scale
$\sigma^*$; the per-point variance $q(x_i^*)$ is then computed with~\cref{eq:noise and variance}. In this example, we set $N=50$.

We compare the following two models:
(1) \textbf{GPC+ (Privileged Model):} $\sigma_i^2 = \sigma_0 \, e^{-\lVert x^*_i\rVert/\tau}$; and (2) \textbf{Baseline:} $\sigma_i^2 \equiv 0$.
Both models are trained with the Gaussian kernel $k(x, y) = \sigma_f^2 \exp(-\lVert x - y\rVert^2 / 2\ell^2)$, where $\ell$ is the lengthscale (a.k.a.\ bandwidth) hyperparameter. Reliable PI assigns a large variance $\sigma^2_i(x_i^*)$ to examples near the noisy boundary and flattens their probit sites; increasing $\sigma^*$ progressively destroys this alignment.
We compare the evidence gap $\Delta Z$ with the clean-test accuracy gain
$\Delta_{\rm acc}\coloneqq\operatorname{Acc}_{\textsc{lupi}}
-\operatorname{Acc}_{\rm never}$.

\begin{figure}[t]
    \centering
    \includegraphics[width=.48\linewidth,trim=0mm 0mm 00mm 0mm,clip]{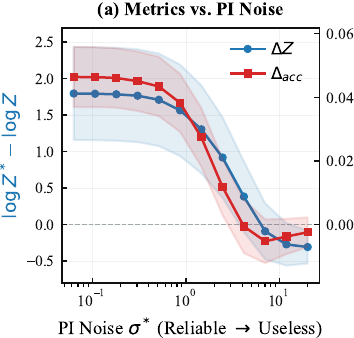}
    \includegraphics[width=.51\linewidth,trim=0mm 0mm 00mm 0mm,clip]{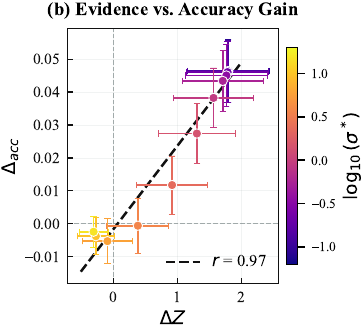}
    \caption{GPC+ on the noisy circular-boundary task, as the privileged feature is
degraded by increasing its noise level $\sigma^*$ (reliable $\to$ useless).
\textbf{(a)} The evidence gap $\Delta_Z = \log Z^* - \log Z$ (blue, left axis)
and the test-accuracy gain $\Delta_{\text{acc}} = \text{Acc}^* - \text{Acc}$
(red, right axis) decay together and vanish once PI carries no useful
information.
\textbf{(b)} The same quantities plotted against each other, coloured by
$\log_{10}(\sigma^*)$.}
\label{fig:gpc-pi}
\end{figure}

\Cref{fig:gpc-pi}(a) shows that $\Delta Z$ and $\Delta_{\rm acc}$ decrease
together and approach zero as the PI channel becomes uninformative for large noise levels.
Both are also highly correlated as depicted in \Cref{fig:gpc-pi}(b).
Hence,
the same training-time diagnostic $\Delta Z$ that predicts PI utility in the
GMM also tracks its realized gain in a supervised, non-conjugate,
function-space model.

\section{Conclusion} \label{sec:conclusion}
We introduced an algorithm-agnostic PAC-Bayesian account of privileged
information (PI) that separates its \emph{potential} from its \emph{realization}
at deployment. The endpoint Gibbs measures turn the former into an exact
partition-function comparison, while the deployment correction quantifies
how much of this advantage survives the removal of PI. Two-sided certificates
further provide sufficient conditions for a genuine separation of true risks,
and the reduction to a fixed-feature subclass of the exponential family extends the analysis beyond canonical Gibbs posteriors. The GMM and privileged-noise GPC experiments show that the same
training-time partition gap tracks realized gains in generative and
discriminative settings.
We highlight that characterizing the slope of this relationship turns the diagnostic from an ordering into a calibrated, quantitative predictor of the achievable gain per unit and remains an open direction for future work.

The present theory assumes data-independent priors and deployment maps, and the stated Catoni certificates require bounded losses; our log-likelihood experiments therefore use the partition gap diagnostically. Computing partition functions may also be difficult in large models, and favorable endpoint potential alone cannot guarantee a good deployment. Future work
should develop unbounded-loss and data-dependent-prior extensions, scalable
evidence estimators, and learning principles for deployment maps that retain
PI-derived information. These directions would turn the proposed criterion
from a model-selection diagnostic into a practical design tool for LUPI.

\newpage

\bibliographystyle{plainnat}
\bibliography{references}

\newpage
\appendix
\section{Proofs} \label{app:proofs}

\subsection{Gibbs Posterior Minimization} \label{proof:lem:gibbs_min}

\begin{proof}[Proof of \cref{lem:gibbs_min}]
Since $\Phi^{-1}_{\lambda/N}$ is strictly increasing, minimizing
$B_\lambda(\rho)$ is equivalent to minimizing
$\lambda\langle r\rangle_\rho+\KL(\rho\Vert\pi)$. From
$\dif\rho_\lambda/\dif\pi=e^{-\lambda r}/Z_\lambda$,
\begin{align*}
\KL(\rho\Vert\rho_\lambda)
&=\E_{h\sim\rho}\!\left[
\log\frac{\dif\rho}{\dif\pi}
-\log\frac{\dif\rho_\lambda}{\dif\pi}
\right]\\
&=\KL(\rho\Vert\pi)+\lambda\langle r\rangle_\rho+\log Z_\lambda.
\end{align*}
Therefore,
\begin{equation*}
\lambda\langle r\rangle_\rho+\KL(\rho\Vert\pi)
=\KL(\rho\Vert\rho_\lambda)-\log Z_\lambda.
\end{equation*}
The right-hand side is uniquely minimized at $\rho=\rho_\lambda$, with value
$-\log Z_\lambda$. Substitution into $B_\lambda$ proves the claim.
\end{proof}

\subsection{Endpoint Certificate Criterion} \label{proof:thm:main}

\begin{proof}[Proof of \cref{thm:main}]
The Catoni bound~\cref{eq:catoni} applied to $\rhoNever$ with $r$ gives
$\langle R\rangle_{\rhoNever} \le B_\lambda$ with probability $1-\varepsilon$. Applied
to $\rhoAlways$ with $r^*$, it bounds
$\langle R\rangle_{\rhoAlways}$;
hence $\langle R\rangle_{\rhoAlways} \le B^*_\lambda$ with probability $1-\varepsilon$.
A union bound gives both simultaneously with probability $1-2\varepsilon$.

Since $\Phi^{-1}_{\lambda/N}$ is strictly increasing, $B^*_\lambda < B_\lambda$
holds if and only if their arguments are ordered; the $\ln(1/\varepsilon)/\lambda$
terms cancel, leaving $\Delta_r > \Delta_{\mathrm{KL}}/\lambda$. Substituting the
Gibbs identities
$\mathrm{KL}(\rhoNever\|\pi) = -\lambda\,\langle r\rangle_{\rhoNever} - \log Z_\lambda$
and
$\mathrm{KL}(\rhoAlways\|\pi^*) = -\lambda\,\langle r^*\rangle_{\rhoAlways} - \log Z^*_\lambda$
yields
\begin{equation}
\label{delta_z_NEW}
    \Delta_r - \frac{\Delta_{\mathrm{KL}}}{\lambda}
    = \frac{\log Z^*_\lambda - \log Z_\lambda}{\lambda}
    = \frac{\Delta Z}{\lambda},
\end{equation}
so the condition reduces to $Z^*_\lambda > Z_\lambda$ since $\lambda > 0$.
\end{proof}

\subsection{Deployment-Certificate Criterion} \label{proof:corollary:deployment-certificate}

\begin{proof}[Proof of \cref{corollary:deployment-certificate}]
Since $\rhoAlways\ll\pi^*$ and $\pi=T_\sharp\pi^*$, we have $\rhoDeployed=T_\sharp\rhoAlways\ll\pi$. Apply~\cref{thm:catoni} on the privileged sample space with loss $(h,x,x^*,y)\mapsto\ell(h(x),y)$. Its empirical and population risks are $r$ and $R$, respectively. The PAC-Bayes event holds
simultaneously for every posterior on $\Hypo$, and therefore yields both inequalities in~\cref{eq:deployed-risk-certificates} without a union bound. Since $\Phi^{-1}_{\lambda/N}$ is strictly increasing, it suffices to compare
the arguments of the two certificates; their confidence terms cancel. The Gibbs identities and~\cref{eq:deployment-gap,eq:kl-contraction} give
\begin{align*}
&\left[\langle r\rangle_{\rhoNever} +\frac{1}{\lambda}\KL(\rhoNever\Vert\pi)\right] -\left[\langle r\rangle_{\rhoDeployed} +\frac{1}{\lambda}\KL(\rhoDeployed\Vert\pi)\right] = \frac{\Delta Z(s^*)-\lambda D_T(s^*)+\kappa_T(s^*)}{\lambda},
\end{align*}
which proves the main claim. Finally, we apply the argument from the data-processing
inequality to prove $\KL(\rhoDeployed\Vert\pi)\leq\KL(\rhoAlways\Vert\pi^*)$. Let $\varphi(t) = t\log t$ be defined for $t>0$, which is a convex function. Considering $\pi = T_\sharp\pi^*$, $\rhoDeployed = T_\sharp\rhoAlways$ and the fact that $\frac{\dif T_\sharp\rhoAlways}{\dif T_\sharp\pi^*}\circ T=\E\left(\frac{\dif\rhoAlways}{\dif\pi^*}\,\middle\vert \,\sigma(T)\right)$, we have
\begin{align*}
    \KL(\rhoDeployed\Vert\pi) &= \int_{\Hypo}\varphi\left(\frac{\dif \rhoDeployed}{\dif\pi}\right)\dif\pi \\
    &=\int_{\Hypo^*}\varphi\left(\frac{\dif \rhoDeployed}{\dif\pi}\circ T\right)\dif\pi^*\\
    &=\int_{\Hypo^*}\varphi\left(\E\left(\frac{\dif\rhoAlways}{\dif\pi^*}\,\middle\vert \,\sigma(T)\right)\right)\dif\pi^* \\
    &\le \int_{\Hypo^*}\E\left(\varphi\left(\frac{\dif\rhoAlways}{\dif\pi^*}\right)\,\middle\vert\,\sigma(T)\right)\dif\pi^* \\
    &=\int_{\Hypo^*}\varphi\left(\frac{\dif\rhoAlways}{\dif\pi^*}\right)\dif\pi^* \\
    &=\KL(\rhoAlways\Vert\pi^*).
\end{align*}
Therefore, $\kappa_T\geq0$, which proves~\cref{eq:sufficient-deployment} and the equality case.
\end{proof}

\subsection{Catoni Lower Certificate} \label{proof:lem:lower_cert}

\begin{proof}[Proof of \cref{lem:lower_cert}]
Apply the Catoni upper bound to the complementary loss $\ell' = 1 - \ell
\in [0,1]$, whose posterior-averaged true and empirical risks are
$1 - \langle R\rangle_\rho$ and $1 - \langle r\rangle_\rho$. With the same
prior $\pi$, temperature $\lambda$, and confidence $\varepsilon$, this gives with
probability $\geq 1 - \varepsilon$
\begin{equation*}
    1 - \langle R\rangle_\rho \leq \Phi^{-1}_{\lambda/N}\!\left(
        1 - \langle r\rangle_\rho + \frac{\mathrm{KL}(\rho \| \pi) + \ln\frac{1}{\varepsilon}}{\lambda}
    \right),
\end{equation*}
and rearranging yields the claim.
\end{proof}

\subsection{Positive Gibbs-Envelope Certificate} \label{proof:thm:true_risk_separation}

\begin{proof}[Proof of \cref{thm:true_risk_separation}]
Apply~\cref{thm:catoni} on $(\Hypo^*,\HypoBorel^*)$ to the PI loss, and apply
\cref{lem:lower_cert} to the ordinary loss on $(\Hypo,\HypoBorel)$. A union
bound proves~\cref{eq:endpoint-two-sided-certificates}. By definition,
$B_\lambda^*=\Phi_{\lambda/N}^{-1}(a^*)$ and $L_\lambda(\rhoNever)=1-\Phi_{\lambda/N}^{-1}(b)$. Hence,
$B_\lambda^* < L_\lambda(\rhoNever)$ iff
\begin{equation*}
\Phi_{\lambda/N}^{-1}(a^*)+\Phi_{\lambda/N}^{-1}(b) < 1\,.
\end{equation*}
Substitution of the definition $\Phi_{\lambda/N}^{-1}(x)=(1-e^{-x\lambda/N})/(1-e^{-\lambda/N})$ proves~\cref{eq:endpoint-separation}. Combining the two valid certificates in~\cref{thm:main,lem:lower_cert} gives~\cref{eq:endpoint-true-risk-separation}. For $x \geq 0$, the inequality $1-e^{-cx}\leq cx$, $c>0$, implies
$\Phi_{\lambda/N}^{-1}(x) \leq \frac{\lambda x}{N(1-e^{-\lambda/N})}$. Therefore, the condition $a^* + b < N(1-e^{-\lambda/N})/\lambda$ implies
\begin{equation*}
\Phi_{\lambda/N}^{-1}(a^*)+\Phi_{\lambda/N}^{-1}(b) \leq \frac{\lambda(a^* + b)}{N(1- e^{-\lambda/N})} \leq1\,.
\end{equation*}
Expanding $a^* + b$ proves~\cref{eq:endpoint-separation-sufficient}.
\end{proof}

\subsection{Deployment True-Risk Separation} \label{proof:corollary:deployment-risk-separation}

\begin{proof}[Proof of \cref{corollary:deployment-risk-separation}]
As above, $\rhoDeployed\ll\pi$. Apply~\cref{thm:catoni} on the augmented
sample space to the ordinary loss, and apply~\cref{lem:lower_cert} there to
the complementary loss. Both losses ignore $x^*$ and have population risks
$R$ and $1-R$ under the $(X,Y)$-marginal $\DataDistri$. The two PAC-Bayes
events are simultaneous over the posterior, so they include $\rhoDeployed$
and $\rhoNever$ even though the former depends on the full sample $S^*$. A
union bound shows that both inequalities in the corollary hold simultaneously
with probability at least $1-2\varepsilon$.

By the definitions of $u$ and $b$,
$B_\lambda(\rhoDeployed)=\Phi^{-1}_{\lambda/N}(u)$ and
$L_\lambda(\rhoNever)=1-\Phi^{-1}_{\lambda/N}(b)$. The same substitution used
for the endpoint measures yields the equivalence stated in the corollary. On
the valid event, certificate
separation gives
\begin{equation*}
\langle R\rangle_{\rhoDeployed} \le B_\lambda(\rhoDeployed) < L_\lambda(\rhoNever)\le\langle R\rangle_{\rhoNever}.
\end{equation*}
Finally, the sufficient condition in the corollary is equivalent to
$u+b<(1-e^{-\lambda/N})/(\lambda/N)$. The inequality
$\Phi_c^{-1}(x)\le cx/(1-e^{-c})$ then proves the stated sufficient criterion.
\end{proof}

\subsection{Fixed-Feature Reduction} \label{proof:lem:reduction}

\begin{proof}[Proof of \cref{lem:reduction}]
Relative to $\pi$, the right-hand side of
\cref{eq:fixed-feature-gibbs} has density proportional to
\begin{equation*}
e^{-\langle\mu,\widetilde\phi(h)\rangle}e^{-\lambda r(h)}
=e^{-\lambda r(h)-\langle\mu,\widetilde\phi(h)\rangle},
\end{equation*}
which is proportional to $\dif\rho/\dif\pi$ by
\cref{eq:fixed_feature}. Both sides are probability measures, so their
normalizing constants agree, and the measures are equal.
\end{proof}

\section{More Examples of the Deployment Map} \label{sec:nontrivial-deployment-examples}
In both validation studies in~\cref{sec:verification}, the deployment map $T$ becomes the identity after hypotheses are identified by the shared parameter $\theta$ (GMM) or latent function $f$ (GPC+). Nontrivial maps arise when the PI model has degrees of freedom that must be removed or averaged out at deployment.

Many LUPI methods jointly learn a deployable branch $f_\theta\colon \X\to\R$ and a correcting or teacher branch $g_\eta\colon\X^*\to\R$ that accesses $x^*$ during
training~\citep{Vapnik2009,LopezPaz2016,HoffmanGD16,Garcia2018}, giving rise to the privileged hypothesis
$$
h_{\theta,\eta}^*\colon \X\times \X^* \to \R^2,\; (x,x^*) \mapsto (f_\theta(x), g_\eta(x^*))\,.
$$
Identifying the hypothesis with the parameter pair $(\theta,\eta)$, the deployment map
is the projection
\begin{equation*}
[Th_{\theta,\eta}^*](x) \coloneqq f_\theta(x).
\end{equation*}

\paragraph{SVM+~\citep{Vapnik2009}.} Given the realization $s^*=(x_i,x_i^*,y_i)$, SVM+ jointly learns the deployable classifier $f_{w,b}(x)=\langle w,x\rangle+b$ and a privileged correcting function $g_{w^*,b^*}(x^*)=\langle w^*,x^*\rangle+b^*$ by solving
\begin{equation*}
\begin{aligned}
\min_{w,b,w^*,b^*}\quad& \tfrac12\|w\|^2+\tfrac{\gamma}{2}\|w^*\|^2 +C\sum_i g_{w^*,b^*}(x_i^*)\\
\mathrm{s.t.} \quad& y_i f_{w,b}(x_i)\geq1-g_{w^*,b^*}(x_i^*), \quad g_{w^*,b^*}(x_i^*)\geq0.
\end{aligned}
\end{equation*}
The privileged branch predicts training-example difficulty but is unavailable at inference. Deployment is therefore the many-to-one projection
$$
T(f_{w,b},g_{w^*,b^*})=f_{w,b}\,.
$$
The PI branch is discarded after training, making $T$ generally many-to-one, which can strictly contract the KL divergence (see~\cref{eq:kl-contraction}), i.e., $\KL(\rhoDeployed\Vert\pi) = \KL(T_\sharp\rhoAlways\Vert T_\sharp\pi^*) < \KL(\rhoAlways\Vert\pi^*)$.

\paragraph{Generalized distillation~\citep{LopezPaz2016}.}
Consider $c$-class classification with $y_i\in\Delta^c$, and let $\sigma\colon\R^c\to\Delta^c$ denote the softmax map. Generalized distillation (GD) works with logit-valued hypotheses, $f_t$ and $f_s$ --- known as the teacher and the student, $\mathcal{F}_t\subseteq\{f_t\colon\mathcal U^*\to\R^c\}$, $\mathcal{F}_s\subseteq\{f_s\colon\X\to\R^c\}$, where $\mathcal U^*=\X^*$ in the original formulation and may be
$\mathcal U^*=\X\times\X^*$ in multimodal variants~\citep{Garcia2018}. The corresponding predictive distributions are $\sigma\circ f_t$ and $\sigma\circ f_s$. Writing $x_i^*\in\mathcal{U}^*$ for the PI, the teacher is first learned
from the privileged examples, $\widehat f_t\in\operatorname*{\arg\,\min}_{f\in\mathcal F_t}\frac1N\sum_{i=1}^N \ell\!\left(y_i,\sigma\circ f(x_i^*)\right) +\Omega(\|f\|)$, where $\Omega$ is a regularizer.
Define the teacher's soft labels $s_i \coloneqq \sigma(\widehat f_t(x_i^*)/\tau)\in\Delta^c$.
The student $\widehat{f}_s$ is then learned on the ordinary inputs by balancing the hard labels with the teacher's soft labels, i.e.,
$
\widehat{f}_s \in \operatorname*{\arg\,\min}_{f\in \mathcal{F}_s} \frac1N\sum_{i=1}^N (1-w) \ell\!\left(y_i,\sigma\circ f(x_i)\right) + w\ell\!\left(s_i,\sigma\circ f(x_i)\right),
$
where $w\in[0,1]$. To fit GD into our framework, we set
$$
\Hypo_{\rm GD}^* \coloneqq \left\{ h_{f_s,f_t}^* \colon (x,x^*) \mapsto (f_s(x), f_t(x^*)) \right\}\,.
$$
Accordingly, we treat the learning rule of GD as
$$
\learner^* \colon ((x_i, x_i^*, y_i))_{i=1}^N \mapsto (\widehat{f}_s,\widehat{f}_t) \in\Hypo_{\rm GD}^*\,.
$$
Note that the learning rule $\learner^*$ here is a frequentist/point estimator, rather than the Bayesian learner we define in~\cref{sec:PAC-Bayes-basics}. To analyze generalized distillation with our results, we need to implement Bayesian inference/estimation for it.
Moreover, the deployment map is the data-independent projection
$$
[T h_{f_s,f_t}^*](x) \coloneqq f_s(x)\,.
$$
Modality-hallucination methods instantiate the same function-space viewpoint: the privileged network produces depth-derived logits or feature maps, while an RGB-input hallucination network is trained to reproduce them and is retained at test time~\citep{HoffmanGD16,Garcia2018}.

\paragraph{Transfer and marginalize (TRAM)~\citep{CollierJK+22}.} Taking the same setup as in generalized distillation, TRAM jointly learns an $x$-only prediction head and a
privileged prediction head over a shared feature space $\mathcal{Z}$. Let
$\phi\colon\X\to\mathcal{Z}$ be the shared feature map,
$\psi\colon\mathcal{Z}\times\X^*\to\mathcal Z^*$ the PI feature map, and define
the logit-valued hypotheses
\begin{align*}
f_{\phi,\theta}(x) &\coloneqq F_\theta(\phi(x)), \quad &F_\theta \colon \mathcal{Z}\to\R^c\,,\\
g_{\phi,\psi,\eta}(x,x^*) &\coloneqq G_\eta\!\left(\psi(\phi(x),x^*)\right), \quad &G_\eta \colon \mathcal{Z}^* \to\R^c\,.
\end{align*}
Their predictive distributions are respectively
$\sigma\circ f_{\phi,\theta}$ and $\sigma\circ g_{\phi,\psi,\eta}$, where $\sigma$ is
the softmax map. The corresponding PI hypothesis is therefore
\begin{equation*}
h_{\phi,\psi,\theta,\eta}^*(x,x^*) \coloneqq \left(f_{\phi,\theta}(x),g_{\phi,\psi,\eta}(x,x^*)\right).
\end{equation*}
Given the realization $s^*=((x_i,x_i^*,y_i))_{i=1}^N$, TRAM jointly learns these two heads by solving
$$
\!\!\min_{\phi,\psi,\theta,\eta}\sum_{i=1}^N \ell\!\left(
y_i,\sigma\!\circ\!F_\theta\!\circ\operatorname{sg}\circ\phi(x_i)\right)+\beta
\ell'\!\left( y_i,\sigma\!\circ\!g_{\phi,\psi,\eta}(x,x^*)\right)\,,
$$
where $\operatorname{sg}$ is the stop-gradient operator: it is the identity
during the forward pass but has zero derivative during backpropagation.
Consequently, the first loss $\ell$ trains the deployable head $\theta$ without modifying $\phi$, whereas the privileged loss $\ell'$ trains $(\phi,\psi,\eta)$. Knowledge from PI is therefore transferred to deployment through the shared representation
$\phi$, rather than through teacher soft labels as in generalized distillation.
Identifying the learned PI hypothesis with $(\phi,\psi,\theta,\eta)$, deployment retains the shared features and the ordinary head:
$$
[T h_{\phi,\psi,\theta,\eta}^*](x) \coloneqq f_{\phi,\theta}(x)\,.
$$
Thus, $T$ discards the PI-specific parameters $(\psi,\eta)$ while retaining the PI-trained feature map $\phi$.

\section{Examples of Fixed-Feature Distributions} \label{sec:fixed-feature-exp-family-example}
We provide more examples of~\cref{def:fixed_feature}
and state the conditions under which the representation is exact. Throughout,
the priors and posteriors are assumed proper and to belong to
$\mathcal M_1^+(\Hypo)$ under the metric fixed in~\cref{sec:background}.

\paragraph{Common likelihood identity.}
Let $\theta\in\Hypo$ have prior $\pi$ and likelihood
$p_\theta(y\mid x)>0$. With
$r_s(\theta)=-N^{-1}\sum_i\log p_\theta(y_i\mid x_i)$, Bayes' rule gives
\begin{equation*}
 \frac{\dif\rho}{\dif\pi}(\theta)
 =\frac{\prod_i p_\theta(y_i\mid x_i)}
 {\int\prod_i p_{\vartheta}(y_i\mid x_i)\pi(\dif\vartheta)}
 =\frac{e^{-Nr_s(\theta)}}{\int e^{-Nr_s(\vartheta)}\pi(\dif\vartheta)}.
\end{equation*}
Thus every proper Bayesian posterior based on an i.i.d.\ likelihood has the
fixed-feature density with $\lambda=N$ and $\mu=0$, provided its marginal
likelihood is finite and nonzero. The following paragraphs verify the named
families and specify what the hypothesis is in each case.

\paragraph{Bayesian linear regression.}
Let $h_\beta(x)=x^\top\beta$, let
$y\mid\beta\sim\mathcal N(X\beta,\sigma^2I_N)$ with fixed $\sigma^2>0$, and
take the data-independent prior $\pi=\mathcal N(m_0,V_0)$. Up to an additive
constant independent of $\beta$,
\begin{equation*}
r(\beta)=\frac{\|y-X\beta\|^2}{2N\sigma^2}, \quad \frac{\dif\rho}{\dif\pi}(\beta)\propto e^{-Nr_s(\beta)}.
\end{equation*}
Completing the square gives
$V_s^{-1}=V_0^{-1}+X^\top X/\sigma^2$ and
$m_s=V_s(V_0^{-1}m_0+X^\top y/\sigma^2)$, so
$\rho=\mathcal N(m_s,V_s)$. Hence the exact posterior, not merely its mean,
has the required density. Unknown noise or hierarchical parameters can be
included in $\theta$ and $\pi$; the common likelihood identity remains valid.

\paragraph{Generalized linear models.}
For a canonical GLM,
$p_\beta(y_i\mid x_i)=\exp(y_ix_i^\top\beta-b(x_i^\top\beta)+c(y_i))$.
Consequently,
\begin{equation*}
r(\beta) = \frac1N\sum_{i=1}^N\left[b(x_i^\top\beta)-y_ix_i^\top\beta-c(y_i)\right],\quad \frac{\dif\rho}{\dif\pi} \propto e^{-Nr}.
\end{equation*}
No conjugacy is needed: logistic, Poisson, and other Bayesian GLMs satisfy the
density identity whenever the posterior is proper. A noncanonical link is
covered by defining $r$ from its actual negative log-likelihood.

\paragraph{Gaussian-process regression.}
Let $\Hypo$ be a Polish path space carrying the data-independent Radon law
$\pi=\mathcal{GP}(m,k)$ (for example, $C(K)$ under the usual sample-continuity
conditions on a compact input domain $K$), and let
$Y_i\mid f\sim\mathcal N(f(x_i),\sigma^2)$ independently. Then
\begin{equation*}
r(f) = \frac{1}{2N\sigma^2}\sum_i(y_i-f(x_i))^2,\quad \frac{\dif\rho}{\dif\pi}(f) \propto e^{-Nr(f)}.
\end{equation*}
The likelihood depends on finitely many measurable evaluations of $f$, so the
Radon--Nikodym derivative is well defined; Gaussian conditioning gives the
usual GP posterior. Its posterior mean agrees with kernel ridge after the
standard identification of the noise-to-prior variance ratio with the ridge
parameter.

\paragraph{Gaussian-process classification.}
With the same GP prior and a Bernoulli likelihood
$p(y_i\mid f)=g(y_if(x_i))$, where $g$ is, for example, the logistic CDF or
$\Phi$, set $r(f)=-N^{-1}\sum_i\log g(y_if(x_i))$. The exact posterior again
satisfies $\dif\rho/\dif\pi\propto e^{-Nr_s}$. It is generally non-Gaussian,
which affects computation but not the density identity. Laplace or
expectation-propagation approximations are separate Gaussian approximations;
they are not automatically exact fixed-feature posteriors.

\paragraph{Common quadratic calculation.}
Let $h_w(x)=\phi(x)^\top w$, $w\in\R^d$, let $F$ have rows
$\phi(x_i)^\top$, and set
$r(w)=\|y-Fw\|^2/(2N)$. For fixed $Q\succ0$ and $\eta>0$, choose
$\pi_Q=\mathcal N(0,(\lambda\eta Q)^{-1})$. Put
$G=F^\top F/N+\eta Q$ and $m=G^{-1}F^\top y/N$. Then
\begin{equation*}
 \rho(\dif w)\propto e^{-\lambda r(w)}\pi_Q(\dif w),
 \qquad \rho=\mathcal N\!\left(m,(\lambda G)^{-1}\right).
\end{equation*}
The mean and mode both minimize $r(w)+\eta w^\top Qw/2$. This calculation
is the common finite-dimensional Gibbs law underlying the next three model
families.

\paragraph{Kernel ridge regression.}
Define the fixed kernel
$k_Q(x,x')=\phi(x)^\top Q^{-1}\phi(x')$ and its training Gram matrix
$K=FQ^{-1}F^\top$. The Gaussian mean above satisfies
\begin{equation*}
Fm=K(K+N\eta I_N)^{-1}y,
\end{equation*}
which is exactly the vector of kernel-ridge fitted values. Thus kernel ridge is
not itself a probability measure; it is the mean and mode of this quadratic
Gibbs posterior. The construction is a valid use of \cref{lem:reduction}
when $\phi,Q,\eta$, and the Gaussian prior are fixed independently of the
PAC-Bayes sample. General fixed kernels are equivalently represented through
their GP priors on functions.

\paragraph{Fixed-shape Gaussian variational posteriors.}
Let
$r(w)=w^\top A_s w/2-b_s^\top w+c_s$, let
$\pi=\mathcal N(0,P_0^{-1})$, and let
$q_s=\mathcal N(m_s,\Sigma_s)$. Direct comparison of Gaussian log-densities
shows that $q_s$ is fixed-feature exactly when
\begin{equation*}
 \Sigma_s^{-1}-P_0-\lambda A_s=M,
 \qquad \Sigma_s^{-1}m_s-\lambda b_s=a,
\end{equation*}
for data-independent $M$ and $a$; these residual natural parameters form the
fixed feature (quadratic and linear sufficient statistics, respectively).
These identities are necessary and sufficient in the nonsingular Gaussian
case. In particular, a prescribed covariance $\Sigma_s\equiv\Sigma$ is covered
when $A_s$ is fixed and the mean natural parameter obeys
$\Sigma^{-1}m_s-\lambda b_s=a$. An arbitrary covariance optimized from the
same sample is not covered unless its variation is exactly the contribution
$\lambda A_s$ already present in the empirical risk.

\paragraph{Frozen last-layer networks.}
Let a feature extractor $\phi_\omega\colon\X\to\R^d$ be fixed independently
of the PAC-Bayes sample and train only
$h_w(x)=w^\top\phi_\omega(x)$ with squared loss and a Gaussian prior on $w$.
Taking $F_{i\cdot}=\phi_\omega(x_i)^\top$ and $Q=I$ in the common quadratic
calculation yields a Gaussian Gibbs posterior whose mean and mode are the
ridge-trained last layer. If $\omega$ was trained on the same sample and then
treated as fixed, the required data independence fails; sample splitting or a
data-dependent-prior theorem is then needed.

\paragraph{NTK-linearized networks.}
Fix an initialization $\vartheta_0$ independently of the sample and linearize
$f_\vartheta$ as
$f_{\vartheta_0}(x)+J_{\vartheta_0}(x)^\top u$, where
$u=\vartheta-\vartheta_0$. After subtracting the fixed offset, this becomes a linear model with design rows $J_{\vartheta_0}(x_i)^\top$. A Gaussian prior
on $u$ and squared loss therefore give the preceding Gaussian Gibbs law, while
the Gram matrix $JJ^\top$ is the empirical neural tangent kernel. Its mean and
mode coincide with NTK kernel ridge. The claim does not extend to a changing
Jacobian or to a linearization point fitted on the same sample without an
additional data-dependence argument.

\section{Gaussian Mixture Details} \label{sec:gmm-details}
We give the construction underlying \cref{sec:GMM}. For
$x\in\R^d$, the data-generating density is
$$
p_\theta(x)=w\mathcal{N}(x;\mu_1,\sigma_1^2I_n)+(1-w)\mathcal{N}(x;\mu_2,\sigma_2^2I_n),
$$
with
$\theta=(w,\mu_1,\mu_2,\sigma_1^2,\sigma_2^2)\in\Theta$. The goal is
posterior inference for $\theta$ and evaluation of the deployed posterior
predictive, not prediction of the latent assignments themselves. PI consists
of the generating assignments $z_i\in\{1,2\}$ observed during training. We write $w_1 = w$, $w_2 = 1-w$.

In both cases we use independent conjugate priors over the mixing weight and
the component parameters,
\begin{equation}
\pi(\theta) = \pi(w)\prod_{k=1}^2 \pi(\mu_k,\sigma_k^2), \;\pi(w)=\mathrm{Beta}(w;1,1)\,,
\end{equation}
with
\begin{align}
\pi(\mu_k,\sigma_k^2) &= \mathrm{NIG}(\mu_k,\sigma_k^2\mid m_0,\kappa,\alpha,\beta) \nonumber\\
&= \mathcal{N}\!\left(\mu_k;m_0,\tfrac{\sigma_k^2}{\kappa}I_n\right)
    \mathrm{Inv\text{-}Gamma}(\sigma_k^2;\alpha,\beta),
\end{align}
where
\begin{equation*}
\mathrm{Inv\text{-}Gamma}(\sigma^2;\alpha,\beta)
 = \frac{\beta^\alpha}{\Gamma(\alpha)}(\sigma^2)^{-\alpha-1}
   \exp\!\left(-\frac{\beta}{\sigma^2}\right)\,,
\end{equation*}
with mean $\E[\sigma^2]=\beta/(\alpha-1)$ for $\alpha>1$. The $\alpha$ parameter
controls the concentration and $\beta$ sets the scale.
Coupling $\mu_k$ to $\sigma_k^2$ through the shared factor
$\sigma_k^2/\kappa$ --- rather than giving $\mu_k$ an independent prior --- is
what makes $\pi(\mu_k,\sigma_k^2)$ conjugate to the Gaussian likelihood: the
posterior is again Normal-Inverse-Gamma with closed-form updates.

\paragraph{Case 1: without privileged information.}
Given a realization of the random sample $s=(x_i)_{i=1}^N$, the posterior is
\begin{equation}
\label{eq:gmm-posterior-noPI}
p(\theta\mid S=s) \propto \pi(\theta) \prod_{i=1}^N\sum_{k=1}^2 w_k\,\mathcal{N}(x_i;\mu_k,\sigma_k^2 I_n).
\end{equation}
Expanding the product of mixtures yields an exact sum over $2^N$
assignments, which is computationally prohibitive as $N$ grows. We therefore
sample the never-PI posterior by Metropolis--Hastings: in each step, a proposal $\theta'$ is accepted
with probability $\min\bigl(1,\ \tfrac{\pi(\theta')\Pr(S = s\mid\theta')}{\pi(\theta)\Pr(S = s\mid\theta)}\bigr)$. The chain has stationary distribution $p(\theta\mid S=s)$, and by the ergodic theorem the empirical distribution converges to the posterior.

\paragraph{Case 2: with privileged information.}
Revealing $z_i$ gives $s^*=\{(x_i,z_i)\}_{i=1}^N$ and collapses the mixture
into a conditional likelihood,
\begin{equation}
p(S^* \mid \theta) = \prod_{i=1}^N w_{z_i} \mathcal{N}(x_i; \mu_{z_i}, \sigma_{z_i}^2 I_n)\,.
\end{equation}
Each point now contributes to exactly one component, so the likelihood is a
single product rather than a product of sums. With
$n_k=\#\{i:z_i=k\}$, the posterior separates into three independent parts:
\begin{align}
&p(\theta\mid S^* = s^*) \propto \underbrace{w^{n_1}(1-w)^{n_2}\,\mathrm{Beta}(w\mid1,1)}_{\text{Beta--Binomial}}\prod_{k=1}^2
\underbrace{\Big[\textstyle\prod_{i:z_i=k}\mathcal{N}(x_i\mid\mu_k,\sigma_k^2 I_n)\Big]
\mathrm{NIG}(\mu_k,\sigma_k^2)}_{\text{Gaussian--NIG (conjugate)}} .
\end{align}
The posterior factorizes completely, enabling exact ancestral sampling:
\begin{align}
w &\sim \mathrm{Beta}(1+n_1,\,1+n_2),\\
\sigma_k^2 &\sim \mathrm{Inv\text{-}Gamma}(\alpha_k',\beta_k'),\\
\mu_k\mid\sigma_k^2 &\sim \mathcal{N}\!\left(m_k',\tfrac{\sigma_k^2}{\kappa_k'}I_n\right).
\end{align}
The updated hyperparameters $(\kappa_k',m_k',\alpha_k',\beta_k')$ follow in
closed form by completing the square in the Gaussian exponent against the NIG
prior, after decomposing
$\sum_{i:z_i=k}\|x_i-\mu_k\|^2 = S_k + n_k\|\mu_k-\bar x_k\|^2$ about the
component sample mean $\bar x_k$. This yields exact i.i.d.\ posterior draws
without MCMC.

Both configurations are represented on the common parameter space $\Theta$; PI enters only through the loss, so the deployment map is the identity, $T=\mathrm{id}_\Theta$, and consequently $\kappa_T=0$ in~\cref{corollary:deployment-certificate}.
The pointwise negative log-likelihood loss functions are
\begin{align*}
&\ell(p_\theta(x_i)) = -\ln\!\Big[w_1 \mathcal{N}(x_i; \mu_1, \sigma_1^2 I_n) + w_2\mathcal{N}(x_i; \mu_2, \sigma_2^2 I_n) \Big]\,, \\
&\ell^*(p_\theta^*(x_i, z_i)) = -\ln w_{z_i}\mathcal{N}(x_i; \mu_{z_i}, \sigma_{z_i}^2 I_n)\,,
\end{align*}
where $w_1=w$, $w_2 = 1-w$, and the empirical risks $r,r^*$ are the sample averages of $\ell$ and $\ell^*$, respectively. The partition functions are evaluated by Monte Carlo over the prior,
\begin{equation} \label{eq:gmm-Z-mc}
Z_\lambda=\E_{\theta\sim\pi}\!\left[e^{-\lambda r(\theta)}\right] \approx\frac1n\sum_{j=1}^n e^{-\lambda r(\theta^{(j)})}, \, \theta^{(j)}\sim\pi,
\end{equation}
and analogously for $Z_\lambda^*$ with $r^*$, giving
$\Delta Z=\log Z^*_\lambda-\log Z_\lambda$. Because the loss function is a negative log-likelihood and hence unbounded, the Catoni certificate applies to a bounded surrogate; we only use $\Delta Z$ here as a diagnostic quantity, as discussed in~\cref{sec:GMM}.

Given the posteriors $\rhoNever$ and $\rhoDeployed$, we take the Bayesian predictive distribution for the probability density at $x$:
$$
\overline{p}_{\text{never}}(x)\coloneqq\!\int_\Theta p_\theta(x) \rhoNever(\dif\theta),\quad\overline{p}_{\textsc{lupi}}(x)\coloneqq\!\int_\Theta p_\theta(x) \rhoDeployed(\dif\theta)\,.
$$
We compare the predictive distribution --- an average over the posterior --- to a baseline of the point-estimate maximum a posteriori (MAP). A MAP estimate $\hat\theta_{\textsc{map}}$ is a point estimator, obtained by maximizing the unnormalized posterior measure:
$$
\hat\theta_{\textsc{map}}=\operatorname*{\arg\,\max}_{\theta\in\Theta} \sum_{i=1}^N\log p_\theta(x_i)+\log\pi(\theta) \,.
$$
A MAP estimate $\hat\theta_{\textsc{map}}$ predicts with the plug-in density $p_{\hat\theta_{\textsc{map}}}(x)$, i.e., the most probable parameter value from the posterior. However, for small $N$, the posterior is often multi-modal, and a MAP only picks a single mode thereof, which leads to an overconfident prediction.

\end{document}